\documentclass{article}
\usepackage[utf8]{inputenc}
\usepackage{hyperref}
\usepackage{graphicx}
\usepackage{caption}
\usepackage{subcaption}
\usepackage{geometry}
\usepackage{amsmath}
\usepackage{amsfonts}
\usepackage{amssymb}
\usepackage{amsthm}
\usepackage{relsize}
\usepackage{nicefrac}
\usepackage{stackrel}
\usepackage[normalem]{ulem}
\usepackage{algorithm}
\usepackage{algpseudocode}
\usepackage{bbm}
\usepackage{bm}

\usepackage{thmtools,thm-restate}
\newtheorem{theorem}{Theorem}[section]

\newtheorem{lemma}[theorem]{Lemma}
\theoremstyle{definition}
\newtheorem{definition}{Definition}[section]
\theoremstyle{remark}

\def\data{\mathcal{D}}

\def\KL{\mathbf{d}_{\mathrm{KL}}}
\def\normal{\mathcal{N}}

\def\E{\mathbb{E}}
\def\V{\mathbb{V}}

\def\Pr{\mathbb{P}}

\def\1{\mathbf{1}}

\newcommand{\ignore}[1]{}

\def\indiWt{\hat{u}^{(k)}_{t-1}}
\def\indiLoss{\hat{\ell}^{(k)}_{t-1}}
\def\indiWtOpt{u^{(k)}_*}

\def\indiLossOpt{\ell^{(k)}_*}
\def\indiWtPrior{\overline{u}}
\def\indiLossPrior{\overline{\ell}}
\def\regVar{\lambda_\ell}
\def\regWt{\Lambda}

\DeclareMathOperator*{\argmin}{arg\,min}
\usepackage{setspace}
\DeclareMathAlphabet{\mathdutchcal}{U}{dutchcal}{m}{n}

\usepackage[mathscr]{euscript}
\hypersetup{
    colorlinks=true,
    linkcolor=blue,
    filecolor=magenta,      
    urlcolor=purple,
}
\usepackage{xcolor}
\newcount\Comments
\newcommand{\kibitz}[2]{\ifnum\Comments=1{\textcolor{#1}{\textsf{\footnotesize #2}}}\fi}

\usepackage{natbib}
\usepackage{authblk}

\title{Adaptive Crowdsourcing Via Self-Supervised Learning}

\author[1]{Anmol Kagrecha}
\author[2]{Henrik Marklund}
\author[1,3]{Benjamin Van Roy}
\author[2]{Hong Jun Jeon}
\author[4]{Richard Zeckhauser}

\affil[1]{Department of Electrical Engineering, Stanford University}
\affil[2]{Department of Computer Science, Stanford University}
\affil[3]{Department of Management Science and Engineering, Stanford University}
\affil[4]{John F. Kennedy School of Government, Harvard University}
\date{}

\begin{document}

\maketitle

\begin{abstract}

Common crowdsourcing systems average estimates of a latent quantity of interest provided by many crowdworkers to produce a group estimate.  We develop a new approach -- {\it predict-each-worker} -- that leverages self-supervised learning and a novel aggregation scheme.  This approach adapts weights assigned to crowdworkers based on estimates they provided for previous quantities.  When skills vary across crowdworkers or their estimates correlate, the weighted sum offers a more accurate group estimate than the average.  Existing algorithms such as expectation maximization can, at least in principle, produce similarly accurate group estimates.  However, their computational requirements become onerous when complex models, such as neural networks, are required to express relationships among crowdworkers.  Predict-each-worker accommodates such complexity as well as many other practical challenges.  We analyze the efficacy of predict-each-worker through theoretical and computational studies.  Among other things, we establish asymptotic optimality as the number of engagements per crowdworker grows.

\end{abstract}

\section{Introduction}

Aggregating opinions from a diverse group often yields more accurate estimates than relying on a single individual.  Applications are wide-ranging.  Intelligence agencies query crowds to help predict uncertain and significant events, showing that a collective effort can give better results \citep{tetlock2016superforecasting}.  Aggregating inputs has also been shown to improve answers to questions in the single-question crowd wisdom problem \citep{prelec2017solution} and in financial forecasting \citep{da2020harnessing}.  Modern artificial intelligence (AI) also relies on aggregation of input from human annotators for the development of image classifiers \citep{vaughan2017making} and chatbots \citep{ouyang2022training, openai2023gpt, touvron2023llama, google2023bard, anthropic2023introducing}.

The recent explosion of AI has led to considerable discussion of situations where computers could partly or wholly replace humans for crowdsourcing tasks \citep{zhu2023can, tornberg2023chatgpt, ollion2023chatgpt, boussioux2023crowdless}. The algorithms we discuss in this paper could apply to such settings but were conceived to apply to groups of people in the spirit of initial work on crowdsourcing.  The mid-nineteenth-century Smithsonian Institution Meteorological Project processed localized data into national weather maps.  One hundred fifty volunteer respondents provided information by telegraph \citep{Smithsonian2024HenryMeteorology}, which was processed by a dozen-plus humans to create a national data map.  By 1860, five hundred local weather stations were reporting \citep{NWS2024Timeline}.

Many traditional crowdsourcing examples involve large numbers of participants.  However, the underlying principles apply even with quite small groups, as would be the norm in many contexts, such as with boards of directors or R\&D research teams.

A conventional approach to aggregation assigns equal importance to each crowdworker's input.  This approach, which we refer to as {\it averaging}, is simple and enjoys wide use.  With a sufficiently large number of crowdworkers, averaging performs well.  However, engaging so many crowdworkers can be costly or infeasible.  Because of this, practical crowdsourcing systems work with a limited number of crowdworkers, and averaging leaves substantial room for improvement.  We discuss two simple cases where averaging can require too many crowdworkers in order to produce accurate estimates.  One arises when skills vary across crowdworkers; those who are more skilled ought to be assigned greater weight.  Another example arises when crowdworkers share sources of information, perhaps because they watch the same news channels or follow similar accounts on social media, or utilize the same data sets.  In this case, each crowdworker's estimate does not provide independent information, and the crowdworker's weight should decrease with the degree of dependence.

In principle, if each crowdworker's level of skill and independence are known, it should be possible to derive weights that improve group estimates relative to averaging.  Our work introduces an approach that \emph{approximates} these weights given \emph{estimates} of past outcomes provided by the same crowdworkers.  The weights reflect what is learned about skills and independence.

If patterns among crowdworker estimates are expressed by simple -- for example, linear -- models, existing algorithms offer effective means for fitting to past observations.  Expectation maximization (EM), in particular, offers a popular option used for this purpose \citep{zhang2016learning, zheng2017truth}.  However, in many contexts, the relationships are complex and call for flexible machine learning models such as neural networks.  Computational requirements of existing algorithms such as EM become onerous.  

To address the limitations of existing methods, we introduce a new approach called \textit{predict-each-worker}, which leverages self-supervised learning and a novel aggregation scheme.  This approach uses self-supervised learning (SSL) \citep{liu2021self} to infer patterns among crowdworker estimates.  In particular, for each crowdworker, predict-each-worker produces an SSL model that predicts the crowdworker's estimate based on past observations and estimates of other crowdworkers.  This enables modeling of complex patterns among crowdworker estimates, for example, by using neural network SSL models.  The aggregation scheme leverages these SSL models to weight crowdworkers to a degree that increases with the crowdworker's skill and independence.

In short, the predict-each-worker approach employs each individual crowdworker’s estimates, after processing by the center, as the basis for producing a group estimate.  These individual estimates represent an atomic approach for building a group estimate.  This approach contrasts sharply with other methods that instead proceed from a series of group estimates.  For example, the famed Delphi method \citep{dalkey1963experimental} involves a panel of experts or crowdworkers who, through several rounds of interactions, refine their estimates.  After each round, a facilitator provides an anonymous summary of the crowdworkers' estimates.  The crowdworkers are then encouraged to revise their earlier estimates in light of the replies of other members of their panel.  This process is repeated until the group converges on a consensus.  This method has been applied, for example, to military forecasting \citep{dalkey1963experimental}, health research \citep{de2003delphi}, and as a method to do graduate research \citep{skulmoski2007delphi}.  While problem settings where crowdworkers could update their estimates are important and interesting, our approach, predict-each-worker, is not designed for such settings, and we leave it for future work to develop an extension.

To investigate the performance of our approach, we perform both computational and theoretical studies.  For a Gaussian data-generating process, we prove that our algorithm is asymptotically optimal, and through simulations, we find that this method outperforms averaging crowdworker estimates.  Moreover, our empirical results indicate that when the dataset is large, predict-each-worker performs as well as an EM-based algorithm.  These studies offer a `sanity check' for our approach and motivate future work for more complex data-generating processes as well as real-world problem settings.

\section{Problem Formulation}
\label{sec:prob_form}

We consider a crowdsourcing system as illustrated in Figure \ref{fig:center}.  In each round, indexed by positive integers $t$, $K$ distinct crowdworkers provide estimates $Y_{t} \in \Re^K$ for an independent outcome $Z_t \in \Re$.  A center observes estimates $Y_t$ to produce a group estimate $\hat{Z}_t \in \Re$ for the outcome $Z_t$.  We assume that the center does not observe the outcomes.

To model uncertainty from the center's perspective, we take $Z_t$ and $Y_t$ to be random variables.  These and all random variables we consider are defined with respect to a probability space $(\Omega, \mathbb{F}, \Pr)$.  We assume that $(Y_t: t \in \mathbb{Z}_{++})$ is exchangeable, and for simplicity, we assume that $\E[Y_t] = 0$.

While we could consider such a process for any sequence of outcomes $Z_t$, we will restrict attention to the case where $Z_t$ is the consensus of an infinite population of crowdworkers, in a sense we now define.  For each $t$, we model crowdworker estimates $(\tilde{Y}_{t,k}: k \in \mathbb{Z}_{++})$ as an exchangeable stochastic process. We assume that the limit $\lim_{K \to \infty} \sum_{k=1}^K \nicefrac{\Tilde{Y}_{t,k}}{K}$ exists almost surely and the consensus is given by
\begin{align}
\label{eq:large_avg_eq_outcome}
Z_t = \lim_{K \to \infty} \sum_{k=1}^K \frac{\Tilde{Y}_{t,k}}{K}.
\end{align}
We define $Z_t$ this way because we believe that such a consensus estimate is often useful -- this is motivated by work on the wisdom of crowds \citep{larrick2006intuitions}.  Another benefit of such a definition is that one can assess performance using observable estimates from out-of-sample crowdworkers without relying on subjective beliefs about the relationship between estimates and outcomes.  The notion of assessment based on out-of-sample crowdworker estimates can be used for hyperparameter tuning, as explained in Appendix~\ref{app:hyperparam_tuning_real}.

We take $Y_t$ to be the first $K$ components $\tilde{Y}_{t,1:K}$.  Hence, each of the $K$ crowdworkers can be thought of as sampled uniformly and independently from the infinite population.  The center's objective is to produce an accurate estimate of the consensus $Z_t$ based on estimates supplied by the finite sample of $K$ crowdworkers.

\begin{figure}
\centering
\includegraphics[scale=0.4]{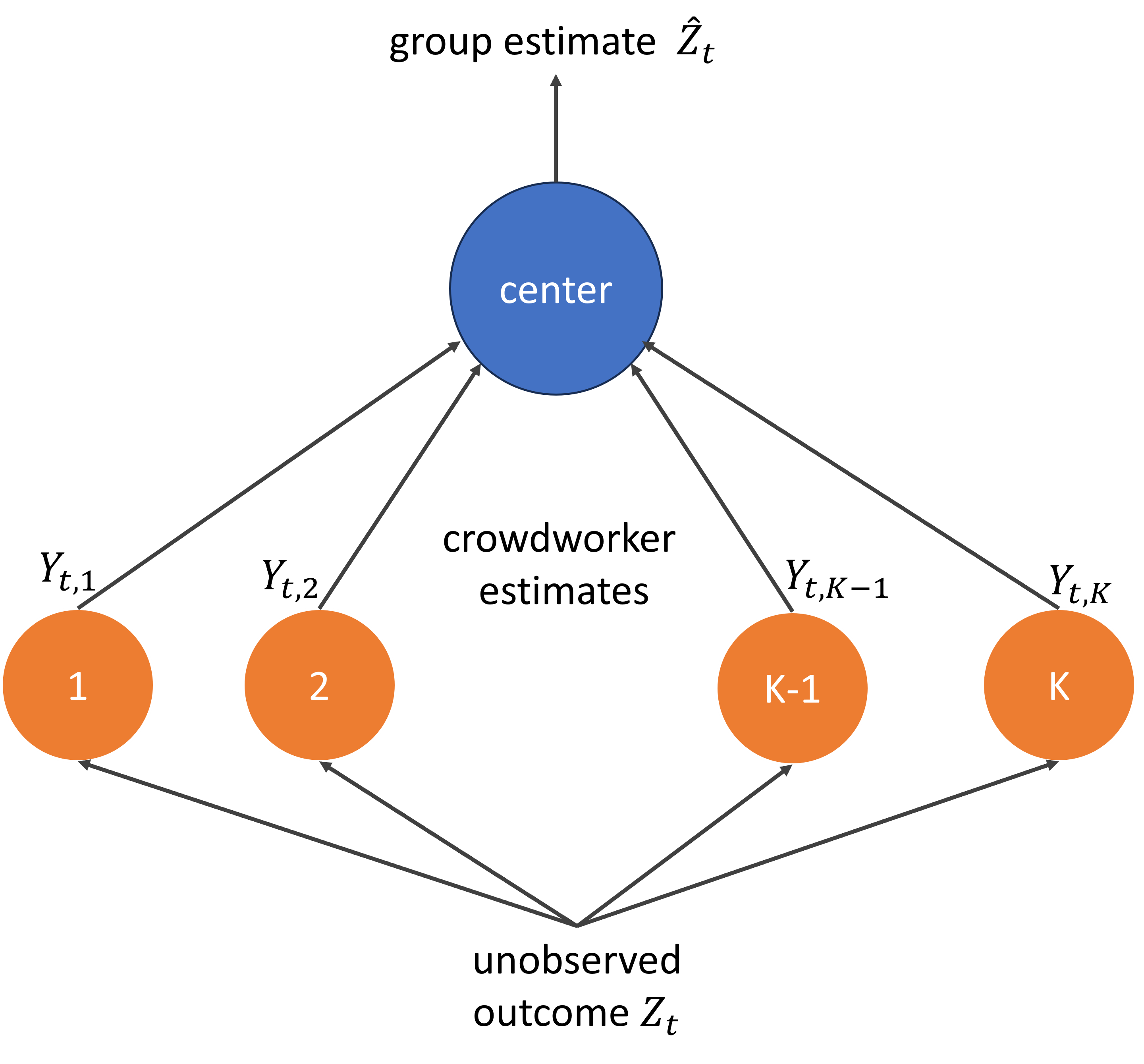}
\caption{Crowdsourcing: each crowdworker provides an estimate of an unobserved quantity $Z_t$, and a center aggregates them to produce a group estimate.}
\label{fig:center}
\end{figure}

Estimates $Y_t$ depend on unknown characteristics of the crowdworkers, such as their skills and relationships.  By de Finetti's Theorem, exchangeability across $t$ implies existence of a random variable $\theta$ such that, conditioned on $\theta$, $(Y_t: t \in \mathbb{Z}_{++})$ is iid.  Intuitively, this latent random variable $\theta$ encodes all relevant information about the crowdworkers and their relative behavior.  While multiple latent variables can serve this purpose, we take $\theta$ to be one that is minimal.  By this we mean that $\theta$ determines the distribution $\Pr(Y_t \in \cdot | \theta)$ and {\it vice versa}.

Before producing a group estimate $\hat{Z}_t$, the center observes past crowdworker estimates $Y_{1:t-1}$ as well as current estimates $Y_t$.  The manner in which the center produces its estimate can be expressed in terms of a {\it policy}, which is a function $\pi$ that identifies a group estimate $\pi(Y_{1:t})$ based on the observed history $Y_{1:t}$ of crowdworker estimates.  When particulars about the policy under discussion matter or when we consider alternatives, we express the dependence of group estimates on the policy using superscripts.  In particular, $\hat{Z}_t^\pi = \pi(Y_{1:t})$.  For each $t$, our objective is to minimize the expected loss $\E[\mathcal{L}(Z_t, \hat{Z}^\pi_t)]$ for some loss function $\mathcal{L}$.  With real-valued group estimates, the squared error $\mathcal{L}(Z_t, \hat{Z}_t^\pi) = (Z_t - \hat{Z}^\pi_t)^2$ offers a natural loss function.

\section{Crowdsourcing Based on Self-Supervised Learning}

This section formally introduces our approach, predict-each-worker. We will do so in three steps.  First, we present the algorithm in an abstract form, which is not efficiently implementable but applies across real-valued estimate distributions.  Then, we describe an implementable instance specialized to a Gaussian data-generating process.  Finally, we explain how our approach extends and scales to complex settings.

\subsection{The General Approach}
\label{sec:jpo_general}

Our approach first learns via self-supervised learning (SSL) to predict each crowdworker's estimate based on the estimates of all other crowdworkers.  The resulting models are then used to produce weights for aggregating crowdworker estimates.  We now describe each of these steps in greater detail.

\subsubsection*{Self-supervised learning}
In the self-supervised learning (SSL) step, we aim to learn patterns exhibited by crowdworkers.  This is accomplished by learning $K$ models, each of which predicts the estimate $Y_{t,k}$ of a crowdworker $k$ given estimates $Y_{t,-k}$ supplied by all other crowdworkers.  Let $\hat{P}^{(k)}_{t-1}$ denote the $k$th model learned from $Y_{1:t-1}$.  Given $Y_{t,-k}$, this model generates a predictive distribution $\hat{P}^{(k)}_{t-1}(\cdot|Y_{t,-k})$ of $Y_{t,k}$.  The intention is for these predictive distributions to accurately estimate a gold standard: the clairvoyant conditional probability $P^{(k)}_{*}(\cdot|Y_{t,-k}) := \Pr(Y_{t,k} \in \cdot | Y_{t,-k}, \theta)$.

As $P^{(k)}_{*}$ can be used to produce a distribution over one crowdworker's estimate given the estimates of all other crowdworkers, clearly, it should depend on patterns among crowdworker estimates.  However, it may not immediately be clear that the set of models $\{P^{(k)}_{*}\}_{k=1}^K$ ought to encode {\it all} observable patterns.  It turns out that this is indeed the case under weak technical conditions.  The following term will be helpful in expressing these conditions.

\begin{definition}[product-form support] 
A probability density function $p$ over $\Re^M$ has {\it product form support} if $\mathrm{support}(p) = \prod_{m=1}^M \mathrm{support}(p_m)$, where $p_m$ is the marginal probability density function of the $m$th component.
\end{definition}
The following theorem establishes sufficiency.
\begin{restatable}[sufficiency of SSL]{theorem}{SSLSufficiency} 
\label{thm:ssl_sufficiency}
If $\Pr(Y_t \in \cdot|\theta)$ is absolutely continuous with respect to the Lebesgue measure and the corresponding density has product-form support, then $\{P^{(k)}_{*}\}_{k=1}^K$ determines $\theta$.
\end{restatable}
While we defer the proof for this theorem to Appendix~\ref{app:proof_ssl_sufficiency}, we now provide some intuition for why it should be true.  It is well-known that Gibbs sampling can be used to sample from a joint distribution based only on its ``leave-one-out" conditionals.  Hence, one could use Gibbs sampling to sample from $\Pr(Y_t \in \cdot | \theta)$ -- the joint distribution -- using the corresponding ``leave-one-out" conditionals $\{P_*^{(k)}\}_{k=1}^K$.  This implies that the set $\{P_*^{(k)}\}_{k=1}^K$ determines $\Pr(Y_t \in \cdot | \theta)$.  Because $\theta$ is minimal, in the sense that $\Pr(Y_t \in \cdot | \theta)$ determines $\theta$, $\{P_*^{(k)}\}_{k=1}^K$ determines $\theta$.

\subsubsection*{Aggregation}
\label{sec:gen_agg}

To compute the aggregation weights, we first compute two quantities:
\begin{enumerate}
    \item \textit{Expected Error}. For each $k$, we compute the expected error $\hat{\ell}^{(k)}_t = \V[Y_{t,k} | Y_{t,-k}, P_{*}^{(k)} \gets \hat{P}_{t-1}^{(k)}]$. The expected error tells us how predictable crowdworker $k$ is.
    \item \textit{SSL Gradient.} For each $k$, we compute the mean $\hat{Y}^{(k)}_t = \E[Y_{t,k} | Y_{t,-k}, P_{*}^{(k)} \gets \hat{P}_{t-1}^{(k)}]$. Then, we compute the gradient of $\hat{Y}^{(k)}_t$ with respect to $Y_{t,-k}$. This gradient $\nabla Y_{t}^{(k)}$ estimates how much any other crowdworker contributed to the predictability of crowdworker $k$.
\end{enumerate}
Finally, we compute the aggregation weights $\hat{\nu}_t \in \Re^K$ as follows. For each $k$,
\begin{align}
    \label{eq:agg_formula}
    \hat{\nu}_{t,k} = \overline{v} \frac{1 - \1^\top \nabla \hat{Y}_{t}^{(k)}}{\hat{\ell}^{(k)}_t},
\end{align}
where $\overline{v} \in \Re_{+}$ is a hyperparameter that indicates the prior variance $\V[Z_t]$ of true outcomes.  If the variance is not known, this hyperparameter can be tuned based on observed data.  We further discuss in Appendix~\ref{app:hyperparam_tuning_real} how to tune $\overline{v}$.  In Section~\ref{sec:jpo_gaussian}, we establish that this algorithm is asymptotically optimal for a Gaussian data-generating process. 

Two principles may help in interpreting the aggregation formula:
\begin{enumerate}
    \item \textit{All else equal, estimate $k$ should be assigned greater weight if it is more predictable.} Consider the SSL gradient and error of estimate $k$. Assume we assign an optimal aggregation weight to this estimate.  Fixing the gradient, if the error were to decrease, this aggregation weight ought to increase because skilled crowdworkers tend be more predictable.  Equivalently, error-prone crowdworkers tend to be less predictable. %
    \item \textit{All else equal, estimate $k$ should be assigned less weight if it is more sensitive to estimate $k'$.} Consider the SSL gradient and error of estimate $k$. Assume we assign an optimal aggregation weight to this estimate.  An increase in the $k'$th component of the gradient indicates increased sensitivity of $k$ to $k'$.  This implies an increase in either the covariance with or skill of $k'$.  In either case, this ought to reduce the aggregation weight. %
\end{enumerate}

\subsection{A Concrete Algorithm for Gaussian Data-Generating Processes}
\label{sec:jpo_gaussian}

In this section, we specialize our algorithm to a Gaussian data-generating process and establish that the algorithm is asymptotically optimal.  Recall that $\theta$ represents a minimal expression of all useful information the center can glean from observing crowdworker estimates $Y_{1:\infty}$.  Let $\Delta_t = Y_t-Z_t \1$ denote the vector of crowdworker errors.  By {\it Gaussian data-generating process}, we mean that $Z_{1:\infty}$ is iid Gaussian, $\Delta_{1:\infty}$ is independent from $Z_{1:\infty}$, and, conditioned on $\theta$, $\Delta_{1:\infty}$ is iid zero-mean Gaussian.  To interpret this data-generating process, it can be useful to imagine first sampling $Z_t$, then for each $k$th crowdworker, adding an independent perturbation $\Delta_{t,k}$ to arrive at an estimate $Y_{t,k} = Z_t + \Delta_{t,k}$.  To simplify our analysis, we will further assume that the covariance matrix $\V[\Delta_t]$ is positive definite.  Our methods and analysis can be extended to relax this assumption. 

For any Gaussian process, the distribution of any specific variable conditioned on others is Gaussian with expectation linear in the other variables.  It follows that, for our Gaussian data-generating process, conditioned on $\theta$, the dependence of any $k$th estimate on others is perfectly expressed by a linear model with Gaussian noise.  In particular, for each $k$, $\theta$ determines coefficients $u_*^{(k)} \in \Re^{K-1}$ and a noise variance $\ell_*^{(k)}$ such that
\begin{align}
\label{eq:linear-SSL-model}
Y_{t,k} = (\indiWtOpt)^\top Y_{t,-k} + \eta_t,
\end{align}
with $\eta_t | \theta \sim \normal(0, \indiLossOpt)$.  This linear dependence motivates use of linear SSL models in implementing our approach to crowdsourcing.  With such models, for each $t$ and $k$, the center would produce estimates $(\hat{u}_t^{(k)}, \hat{\ell}_t^{(k)})$ of $(u_*^{(k)}, \ell_*^{(k)})$.

\subsubsection*{Aggregation with Linear SSL Models}

With linear SSL models, our general aggregation formula \eqref{eq:agg_formula} simplifies.  This is because the gradient with respect to $Y_{t,-k}$ is simply $\indiWt$.  Consequently, the aggregation weights satisfy
\begin{align}
\label{eq:agg_formula_gaussian}
\hat{\nu}_{t,k} = \overline{v} \frac{(1 - \1^\top \indiWt)}{\indiLoss},
\end{align}
where, as defined in Section~\ref{sec:gen_agg}, $\overline{v} = \V[Z_t]$.
We explained intuition in Section~\ref{sec:gen_agg} to motivate our general aggregation formula.  For the special case of a Gaussian data-generating process, it is easy to corroborate this intuition with a formal mathematical result, which we now present.

No policy can produce a better estimate than $\E[Z_t | Y_t, \theta]$, which benefits from knowledge of $\theta$.  In the Gaussian case, it is easy to show this estimate is attained by the weight vector
\begin{align}
\nu_* = \overline{v} S_*^{-1}\1,
\end{align}
where $S_*$ is the covariance matrix of $Y_t$ conditioned on $\theta$.  The following result, which is proved in Appendix~\ref{app:opt_wts_opt_ssl_params}, offers an alternative characterization of $\nu_*$ that shares the form of Equation \ref{eq:agg_formula_gaussian}.
\begin{restatable}{theorem}{optWtsOptSSLParams}
\label{thm:opt_wts_opt_ssl_params}
For the Gaussian data-generating process, for each $k \in \{1,\cdots,K\}$, 
\begin{align*}
\nu_{*,k} = \overline{v}\frac{1 - \1^\top \indiWtOpt}{\indiLossOpt}.
\end{align*}
\end{restatable}
\noindent From this result, we see that if the linear SSL model parameters $(\hat{u}_t^{(k)}, \hat{\ell}_t^{(k)})$ converge to $(u_*^{(k)}, \ell_*^{(k)})$ then the aggregation weights $\hat{\nu}_t$ converge to $\nu_*$.

\subsubsection*{Fitting Linear SSL Models}

The center can generate the estimates $(\hat{u}_t^{(k)}, \hat{\ell}_t^{(k)})$ of coefficients and variances via Bayesian linear regression (BLR). 
 This is accomplished by minimizing the loss function
\begin{align}
\label{eq:ssl_loss}
    f^{(k)}_{Y_{1:t-1},\regWt,\regVar,\indiWtPrior,\indiLossPrior}(u, \ell) 
    = & \underbrace{\sum_{\tau < t} \frac{(Y_{\tau, k} - u^\top Y_{\tau, -k})^2}{2 \ell} + \frac{\log(\ell)}{2}}_{\mathrm{log\ likelihood}} + \underbrace{ \phi_{\regWt, \indiWtPrior}(u, \ell) + \varphi_{\regVar, \indiLossPrior}(\ell)}_{{\rm regularization}},
\end{align}
to obtain an estimate $(\indiWtOpt, \indiLossOpt)$.  The regularization term can be interpreted as inducing a prior, which is specified by four hyperparameters: $\regWt \in \mathcal{S}^K_{+}$, $\regVar \in \Re_{+}$, $\indiWtPrior \in \Re$, $\indiLossPrior \in \Re_{+}$. The term $\phi_{\regWt, \indiWtPrior}(\cdot, \ell)$ is the negative log density of a Gaussian random variable with mean $\indiWtPrior \1$ and covariance matrix $\ell \Lambda^{-1}$, while $\varphi_{\regVar, \indiLossPrior}(\cdot)$ is the negative log density of an inverse gamma distribution with shape parameter $\nicefrac{\regVar}{2}$ and scale parameter $\nicefrac{(\regVar + K + 1) \indiLossPrior}{2}$. Note that each of the $K$ models shares the same hyperparameters. This means that, before the center has observed any data, each crowdworker will receive equal aggregation weight.

While BLR produces the posterior distribution over weights for each of the $K$ models, it ignores interdependencies across models.  In the asymptotic regime of large $t$, this is inconsequential because the posterior concentrates.  However, when $t$ is small, this imperfection induces error.  To ameliorate this error, we introduce a form of aggregation weight regularization.  Let $\Tilde{\nu}_1 \in \Re^K$ be a vector of equal weights produced by our algorithm if applied before learning from any data.  For any $t$, we generate aggregation weights by taking a convex combination $\hat{\nu}_{t,k} = \gamma_t \Tilde{\nu}_{1,k} + (1 - \gamma_t) \Tilde{\nu}_{t,k}$, where $\Tilde{\nu}_{t,k}$ is the vector of unregularized aggregation weights. The parameter $\gamma_t$ decays with $t$. Specifically, $\gamma_t = \nicefrac{r}{(r+t-1)}$, where  $r$ is a hyperparameter governs the rate of decay.  All steps we have described to produce an estimate $\hat{Z}_t$ are presented in Algorithm~\ref{alg:jpo_gaussian}.

\begin{algorithm}
\caption{Predict-each-worker for Gaussian data generating process} \label{alg:jpo_gaussian}
\begin{algorithmic}[1]
\Procedure{PredictEachWorkerGaussian}{$Y_{1:t}$, $\regWt$, $\regVar$, $\indiWtPrior$, $\indiLossPrior$, $r$}
  \State // SSL
    \For{$k \in \{1,\cdots,K\}$}
        \State $\indiWt, \indiLoss \in \argmin_{u \in \Re^{N-1}, \ell \in \Re_+} f_{Y_{1:t-1}, \regWt, \regVar, \indiWtPrior, \indiLossPrior}^{(k)}(u, \ell)$
    \EndFor
  \State 
  \State // Aggregation
    \State $\gamma = \nicefrac{r}{r+t-1}$
    \For{$n \in \{1,\cdots,K\}$}
        \State $\tilde{\nu}_{t,k} \gets \nicefrac{(1 - \mathbf{1}^\top \indiWt)}{\indiLoss}$ 
        \State $\tilde{\nu}_{1,k} \gets \nicefrac{(1 - (K-1) \overline{u})}{\overline{\ell}}$
        \State $\hat{\nu}_{t,k} \gets \gamma \Tilde{\nu}_{1,k} + (1-\gamma) \Tilde{\nu}_{t,k}$
    \EndFor
    \State $\hat{Z}_t \gets \left(\hat{\nu}_t\right)^\top Y_t$
    \State
    \State \textbf{Return} $\hat{Z}_t$
\EndProcedure
\end{algorithmic}
\end{algorithm}

The following result formalizes our claim of asymptotic optimality.
\begin{restatable}{theorem}{ConsistencyThm}
\label{thm:consistency_thm}
If $\Lambda \succ 0$, $\indiLossPrior>0$, and $\regVar \geq 0$, then, for any $\tau \in \mathbb{Z}_{++}$,
\begin{align*}
\lim_{t \to \infty} \Big|\underbrace{\hat{\nu}_t^\top Y_t}_{\mathrm{estimate}}  - \underbrace{\E[Z_t | Y_{1:\infty}]}_{\mathrm{clairvoyant}}\Big| \overset{{\rm a.s.}}{=} 0.
\end{align*}
\end{restatable}
We defer the proof to Appendix~\ref{app:consistency_thm}. While this theorem establishes asymptotic optimality, we empirically study finite-data performance in Section~\ref{sec:sim_study}.

\subsection{Extensions to Complex Settings using Neural Networks} 
\label{sec:nn_exts}

Before conducting an empirical study with the Gaussian data-generating process, we first discuss how our general approach extends beyond the Gaussian data-generating process.  In particular, with neural network SSL models, our approach can capture complex characteristics of and relationships among crowdworkers.  Moreover, with neural network models, our approach extends to accommodate scenarios with missing estimates or where additional context is provided, for example, a language model prompt or crowdworker profile.

Concretely, this process involves training $K$ neural networks, where the $k$th network predicts the $k$th crowdworker using others' estimates.  Additional information, such as a context vector, can just be appended to the neural network's input.  Training can be done via standard supervised learning algorithms, such as stochastic gradient descent (SGD). 

To reduce computational demands, rather than training $K$ separate models, we can train one neural network with $K$ so-called output heads.  The $k$th head would output an estimate of the mean and variance of the $k$th crowdworker. The input to the neural network is the vector of estimates with the $k$th crowdworker masked out.  When training, we would use SGD and randomly choose which crowdworker to mask at each SGD step. This training procedure, where we try and predict a masked out component, is prototypical of self-supervised learning.

\subsubsection*{Incorporating Context}

In practice, crowdworkers respond to specific inputs or contexts, such as images or text from a language model.  To extend to such settings, we can transform the context into a vector.  Such a transformation is naturally available for images. With text, we can use embedding vectors, which can be generated using established techniques \citep{kenter2015short, mars2022word}.  By appending such context vectors to crowdworker estimates, the SSL step can learn context-specific patterns.  For example, if we learn that certain crowdworkers are more skilled for particular types of inputs, it should be possible to produce better group estimates.

\subsubsection*{Including Crowdworker Metadata}

Metadata about crowdworkers, like their educational background or expertise, can be used to improve group estimates.  For example, if a crowdworker's metadata takes the form of a feature vector, our approach can leverage this data.  By concatenating crowdworker estimates with their feature vectors, the neural network can then account for their distinguishing characteristics and not merely the estimates that they provide.  This can prove particularly beneficial when managing a transient crowdworker pool, as it allows us to generalize quickly to previously unseen crowdworkers.

\subsubsection*{Addressing Missing Estimates}

With a sizable pool of crowdworkers, each task might only be assigned to a small subset. 
 As such, each crowdworker may or may not supply an estimate for each $t$th outcome.  To learn from incomplete data gathered in this mode, we can use a neural network architecture where the input associated with each crowdworker indicates whether the crowdworker has supplied an estimate, and if so, what that estimate is.  Through this approach, the neural network can discern intricate patterns and interdependencies among the crowdworkers, such as when the significance of one's input may be contingent on the absence of another due to overlapping expertise.

\subsubsection*{Aggregating the Predictions}

The general aggregation formula presented in Section~\ref{sec:gen_agg} applies to neural network models.  For each crowdworker $k$, we first use the neural network to estimate the mean and variance of their prediction.  Then, we compute the gradient of the mean with respect to $Y_{t,-k}$. Using the estimated variance and the gradient, \eqref{eq:agg_formula} produces aggregation weights.

We primarily designed our approach for problem settings where outcomes are not observed. However, this approach can be extended to settings where some or all outcomes are observed. Please see Appendix~\ref{app:ext_obs_outcomes} for a brief discussion.

\section{A Simulation Study}
\label{sec:sim_study}

To understand how our algorithm compares to benchmarks, we perform a simulation study. We use a Gaussian data-generating process and three benchmarks: averaging, an EM-based policy, and the clairvoyant policy.  The section first introduces the data-generating process, followed by the three policies, and finally presents the simulation results.

\subsection{A Gaussian Data-Generating Process}

In a prototypical crowdsourcing application, some crowdworkers may be more skilled than others. Further, because the crowdworkers may share information sources, their errors may end up being correlated. We consider a model that is simple yet captures these characteristics.  For each crowdworker $k \in \mathbb{Z}_{++}$, this model generates a sequence of estimates according to
\begin{align*}
\Tilde{Y}_{t,k} = \tilde{Z}_t + \Tilde{\Delta}_{t,k}.
\end{align*}
where $\tilde{Z}_t \sim \normal(0,1)$, and $\Tilde{\Delta}_{t,k}$ is additive noise that is independent from $\tilde{Z}_t$. The additive noise is generated as follows:
\begin{align*}
\Tilde{\Delta}_{t,k} = \sum_{n=1}^N C_{k,n} X_{t,n},
\end{align*}
where $(X_t: t \in \mathbb{Z}_{++})$ is an iid sequence of standard Gaussian vectors, and $C_k \in \Re^N$ modulates the influence of $X_t$ on $\tilde{Y}_{t,k}$.  It may be helpful to think of each component $X_{t,n}$ as a random factor and the coefficient $C_{k,n}$ as a factor loading.

For each crowdworker $k \in \mathbb{Z}_{++}$ and factor index $n \in \{1, \dots, N \}$, we sample $C_{k,n} \sim \normal(0,n^{-q})$. Each $C_{k,n}$ is sampled independently from $\tilde{Z}_t$, the factor $X_t$, and any other coefficient $C_{k',n'}$, where $(k', n')\neq (k,n)$. %

We show in Appendix~\ref{app:data_gen_prob_form_asmp} that this data-generating process satisfies the three assumptions of our general problem formulation, as presented in Section~\ref{sec:prob_form}. 
 In particular, $(\tilde{Y}_t: t \in \mathbb{Z}_{++})$ is exchangeable, and for each $t$, $(\tilde{Y}_{t,k}: k \in \mathbb{Z}_{++})$ is exchangeable, with a sample mean that converges almost surely.  It is easy to verify that the sample mean $Z_t = \lim_{K \rightarrow \infty} \sum_{k=1}^K \tilde{Y}_{t,k} / K$ is almost surely equal to $\tilde{Z}_t$.

Recall that we denote by $\theta$ a minimal random variable conditioned on which $(\tilde{Y}_t: t \in \mathbb{Z}_{++})$ is iid.  For any $K$, conditioned on $\theta$, $\Tilde{\Delta}_{t,1:K}$ and $\Tilde{Y}_{t,1:K}$ are distributed Gaussian. It is helpful to consider the covariance matrix of $\Tilde{\Delta}_{t, 1:K}$ conditioned on $\theta$, which we denote by $\Sigma_*$.  Each diagonal element of $\Sigma_*$ expresses the reciprocal of a crowdworker's skill.  Off-diagonal elements indicate how noise covaries across crowdworkers.  

At this point, this data-generating process can feel abstract and non-intuitive.  In Section~\ref{sec:benchmarks}, after we define a few benchmarks, we will provide some intuition about our data-generating process using these benchmarks.

\subsection{Benchmarks}
\label{sec:benchmarks}

We compare our method against three other methods: averaging, a clairvoyant policy, and an EM-based policy.

\subsubsection*{Averaging}

Our first benchmark averages across crowdworker estimates to produce a group estimate $\pi(Y_{1:t}) = \frac{1}{K} \sum_{k=1}^K Y_{t,k}$, where $Y_t=\tilde{Y}_{t,1:K}$. For the Gaussian data-generating process, group estimates produced in this manner converge to $Z_t$ almost surely as $K$ grows. 

\subsubsection*{Clairvoyant Policy}

The clairvoyant policy offers an upper bound on the performance. This policy has access to $\theta$, which encapsulates all useful information that can be garnered from any number of observations $Y_1, Y_2, \ldots$.  Equivalently, the policy has access to $\Pr(Y_t \in \cdot | \theta)$, and it produces a group estimate $\E[Z_t | Y_t,\theta]$. For the Gaussian data-generating process, this means having access to noise covariance matrix $\Sigma_*$ introduced earlier.  In particular, the group estimate can be written as $\E[Z_t | Y_t,\theta]=\E[Z_t | Y_t,\Sigma_*]=\nu_*^\top Y_t$.  This attains mean-squared error 
\begin{align*}
\E[(Z_t - \nu_*^\top Y_t)^2] = \E\left[\frac{1}{1 + \1^\top \Sigma_*^{-1} \1}\right].
\end{align*}
This level of error is lower than what is achievable by any implementable policy, which would not have the privileged access to $\Sigma_*$ that is granted to the clairvoyant policy.  Yet this represents a useful target in policy design as this level of error is approachable as $t$ grows.  

Having defined averaging and clairvoyant policies, we will now provide some perspective on the benefits of clairvoyance or, equivalently, learning from many interactions with the same crowdworkers.  In Figure~\ref{fig:others_vs_sa}, we compare these two policies and a third, which we refer to as only-skills-clairvoyant.  As the name suggests, in order to generate aggregation weights, only-skills-clairvoyant uses diagonal elements (inverse-skill) of $\Sigma_*$ but not the off-diagonal elements (error covariances).  The clairvoyant policy offers the greatest benefit over averaging, though only-skills-clairvoyant also offers some benefit.  For example, to match the performance attained by averaging over one hundred crowdworkers, the clairvoyant policy requires about twenty crowdworkers, while only-skills-clairvoyant needs about seventy. 
 Moreover, the concavity of the plot indicates that the percentage reduction afforded by clairvoyance in the number of crowdworkers required grows with the desired level of performance.  Hence, for a data-generating process like ours, it is beneficial to learn leverage patterns among crowdworker estimates.

\begin{figure}[htp]
    \centering
    \includegraphics[width=0.5\textwidth]{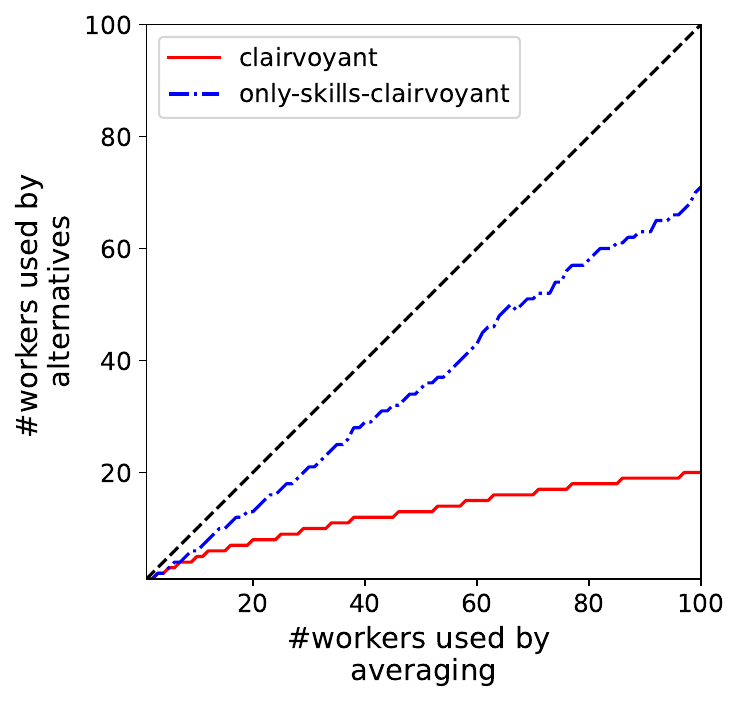}
    \caption{The number of crowdworkers required by clairvoyant and only-skills-clairvoyant policies to match the performance of averaging over various numbers of crowdworkers.  These results were generated using the Gaussian data-generating process with $N=1000$ factors and factor concentration parameter $q=1.7$.  Performance was measured in terms of mean-squared error.}
  \label{fig:others_vs_sa}
\end{figure}

\subsubsection*{EM-Based Policy}

We now discuss a policy based on the EM framework. At a high level, EM iterates between estimating the covariance matrix $\Sigma_*$ and estimating past outcomes $Z_{1:t-1}$. Upon termination of this iterative procedure, the estimate $\hat{\Sigma}$ of $\Sigma_*$ is used to produce a group estimate $\E[Z_t | Y_t,\Sigma_* \leftarrow \hat{\Sigma}]$.  Here, $\Sigma_* \leftarrow \hat{\Sigma}$ indicates that the expectation is evaluated as though $\hat{\Sigma}$ were the realized value of $\Sigma_*$.

Each iteration of the EM algorithm proceeds in two steps.  In the first step, the algorithm %
computes values $\hat{z}_{1:t-1}$ and $\hat{v}_{1:t-1}$, which represent estimates of $\E[Z_{1:t-1} | Y_{1:t-1}]$ and $\V(Z_{1:t-1}| Y_{1:t-1})$, respectively.  These estimates are generated as though the estimated covariance matrix $\hat{\Sigma} \in \mathcal{S}^K_{++}$, from the previous iteration, is $\Sigma_*$; in particular, $\hat{z}_{1:t-1} \leftarrow \E[Z_{1:t-1} | Y_{1:t-1}, \Sigma_* = \hat{\Sigma}]$ and $\hat{v}_{1:t-1} \leftarrow \V[Z_{1:t-1} | Y_{1:t-1}, \Sigma_* = \hat{\Sigma}]$.  The initial estimate of $\Sigma_*$ is taken to be $\overline{\Sigma}$, which is a hyperparameter that can be interpreted as the prior mean for $\Sigma_*$. 

In the second step, estimates $\hat{z}_{1:t-1}$ and $\hat{v}_{1:t-1}$, from the previous step, are used to compute an approximation of the log posterior of $\Sigma_*$. This approximation $g_{y,\overline{\Sigma},c,\hat{z}, \hat{v}}(\hat{\Sigma})$ is known as the evidence lower bound (ELBO). The goal is to find the matrix $\hat{\Sigma}$ that maximizes the ELBO. For the Gaussian data-generating process, in particular, the ELBO is 
\begin{align*}
g_{y,\overline{\Sigma},c,\hat{z}, \hat{v}}(\hat{\Sigma})
=& 
- (c + 2K + 1 + t) \log(|\hat{\Sigma}|)-\text{trace}(c \overline{\Sigma} \hat{\Sigma}^{-1}) \\
&- \sum_{\tau=1}^{t-1} \left((y_\tau - \hat{z}_\tau \1 )^\top \hat{\Sigma}^{-1} (y_\tau - \hat{z}_\tau \1) + \hat{v}_{\tau}\1^\top \hat{\Sigma}^{-1} \1 +2 \mathbf{d}_{\mathrm{KL}}\left(\normal(\hat{z}_\tau, \hat{v}_\tau)||\normal(0,1)\right) \right),
\end{align*}
where $\mathbf{d}_{\mathrm{KL}}\left(\normal(\hat{z}_\tau, \hat{v}_\tau)||\normal(0,1)\right)$ is the Kullback-Leibler divergence between
distributions $\normal(\hat{z}_\tau, \hat{v}_\tau)$ and $\normal(0, 1)$. The scalar $c$ is a hyperparameter that expresses the concentration of the prior distribution of $\Sigma_*$. It is easy to verify that the maximizer of $g_{y,\overline{\Sigma},c,\hat{z}, \hat{v}}$ admits a closed-form solution, and this maximizer is stated in step~\ref{map_via_em:compute_sigma} of Algorithm~\ref{alg:map_via_em}.  %

The iterative EM process continues until either of two stopping conditions is satisfied.  The first stopping condition is triggered when consecutive estimates $\hat{z}_{1:t-1}$ and $\hat{z}'_{1:t-1}$ are sufficiently similar, in the sense that $\| \hat{z}_{1:t-1} - \hat{z}_{1:t-1}' \|^2/(t-1) < \epsilon$.  This is a natural stopping condition, as it represents a measure of near-convergence.  The second stopping condition places a cap on the number of iterations.  This prevents the algorithm from looping indefinitely. The EM steps are presented concisely in Algorithm \ref{alg:map_via_em}.

\subsubsection*{Selecting EM Hyperparameters}

In our experiments, we set $\epsilon=10^{-10}$ and $M=10,000$. For each value of $K$ and $t$, we perform a grid search to choose the values of $\overline{\Sigma}$ and $c$. For each value in the grid search, we run the EM algorithm over multiple seeds, where each seed corresponds to one train set. We choose the values of $\overline{\Sigma}$ and $c$ that result in the highest estimate of MSE (estimation procedure discussed in Appendix~\ref{app:mse_aw}). We let $\overline{\Sigma}$  be a matrix with diagonal elements $\overline{\sigma}^2$ and correlation $\overline{\rho}$. We vary $\overline{\sigma}^2 \in \{ 0.2, 2, 20 \}$ and $\overline{\rho} \in \{ 0, 0.1 \}$. We vary $c \in \{0.1, 1, 10 \}$.

\begin{algorithm}
\caption{expectation-maximization-based policy} \label{alg:map_via_em}
\begin{algorithmic}[1]
\Procedure{EM\_Aggregation}{$Y_{1:t}$, $\overline{\Sigma}$, $c$, $\epsilon$, $M$}
\State $\hat{\Sigma} \gets \overline{\Sigma}$ \label{map_via_em:sigma_init}
\State $\hat{z} \gets \infty$
\State $m \gets 0$
\Repeat
\State $\hat{z}' \gets \hat{z}$
\For{$\tau \in \{1,\cdots,t-1\}$}
\State $\hat{z}_\tau \gets \frac{\1^\top \hat{\Sigma}^{-1} Y_\tau}{1 + \1^\top\hat{\Sigma}^{-1}  \1}$ \label{map_via_em:zbar}
\State $\hat{v}_\tau \gets \frac{1}{1 + \1^\top \hat{\Sigma}^{-1} \1}$ \label{map_via_em:vbar}
\EndFor
\State $\hat{\Sigma} \gets \frac{c \overline{\Sigma} + \sum_{\tau<t} (Y_\tau - \hat{z}_\tau)(Y_\tau - \hat{z}_\tau)^\top + \sum_{\tau<t} \hat{v}_\tau \1\1^\top}{c + 2K + t + 1}$ \label{map_via_em:compute_sigma}
\State $m \gets m+1$
\Until{$\| \hat{z}' - \hat{z} \|^2/(t-1) < \epsilon$ or $m = M$}
\State $\hat{Z}_t \gets \left( \hat{\Sigma} + \1\1^\top \right)^{-1} Y_t$
\State $\text{Return } \hat{Z}_t$
\EndProcedure
\end{algorithmic}
\end{algorithm}

\subsection{Performance Comparisons}

We now compare our method to the three benchmarks described before.  We do so using the Gaussian data-generating process, with $N=1000$ factors and factor concentration parameter $q=1.7$.  This choice of parameters results in expected noise variance $\E[\Sigma_{*,k,k}] \approx 2$.   For each policy $\pi$, we use mean squared error, defined $\E[(Z_t - \hat{Z}_t^\pi)^2]$, as the evaluation metric.

To study how robust our algorithm is to the number of crowdworkers and the dataset size, we evaluate performance for $K \in \{ 10, 20, 30 \}$ crowdworkers and multiple dataset sizes $t \in \{1, K, 10 \times K, 100 \times K, 1000 \times K \}$. We plot results in Figure~\ref{fig:perf_comp}.

We find that predict-each-worker, our algorithm outperforms averaging for all $K$ and $t$.  Moreover, our algorithm performs equally well compared to the clairvoyant, as well as EM policies, for large $t$.  This corroborates Theorem \ref{thm:consistency_thm}.  For smaller values of $t$, EM offers some advantage, though, in the extreme case of $t=1$, our algorithm again is competitive with EM. The details of how the hyperparameters of predict-each-worker are tuned can be found in Appendix~\ref{app:jpo_tuning}.

Our results demonstrate that, like EM, predict-each-worker performs as well as possible as the dataset grows.  Another important feature, as discussed in Section~\ref{sec:nn_exts}, is that predict-each-worker can learn from complex patterns using neural network models.  Taken together, these two features make predict-each-worker an attractive choice relative to EM.

\begin{figure}[htp]
    \begin{minipage}{0.45\textwidth}
        \centering
        \includegraphics[width=\textwidth]{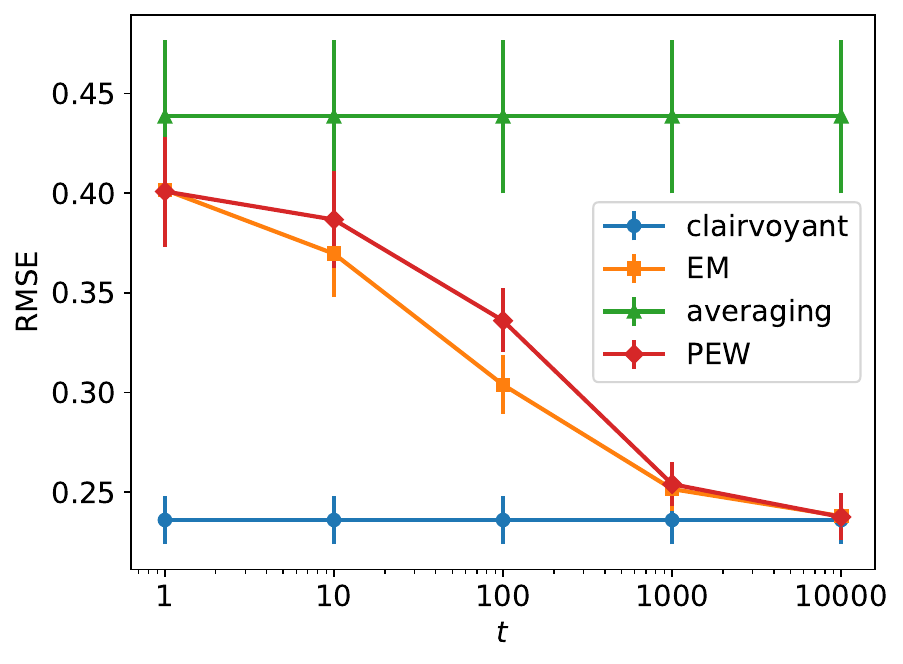}
        \subcaption{$K=10$}
        \label{fig:k_10_large}
    \end{minipage}
      \hfill
    \begin{minipage}{0.45\textwidth}
        \centering
        \includegraphics[width=\textwidth]{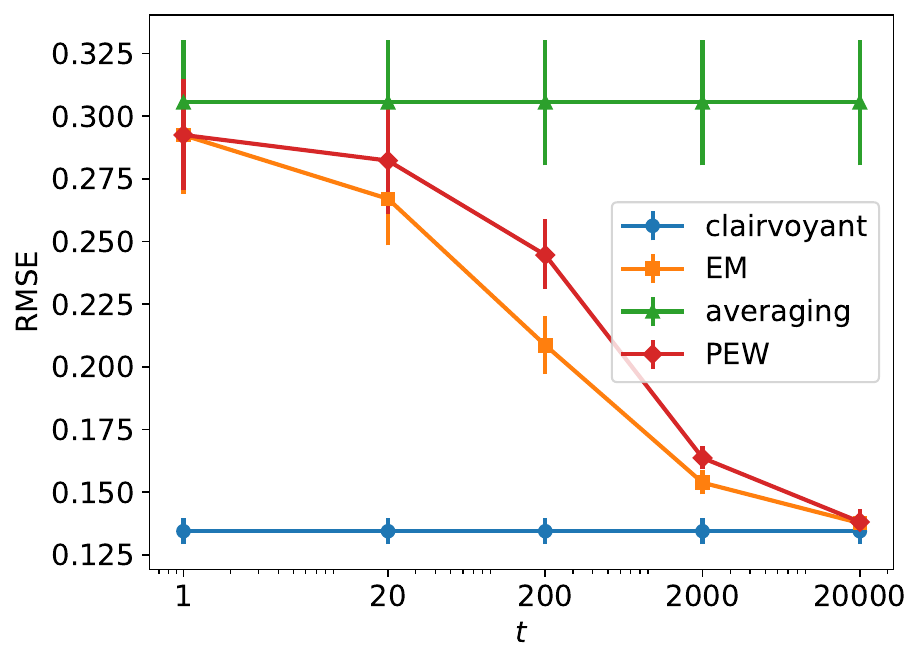}
        \subcaption{$K=20$}
        \label{fig:k_20_large}
    \end{minipage}

    \vspace{10pt}
    \centering
    \begin{minipage}{0.45\textwidth}
        \centering
        \includegraphics[width=\textwidth]{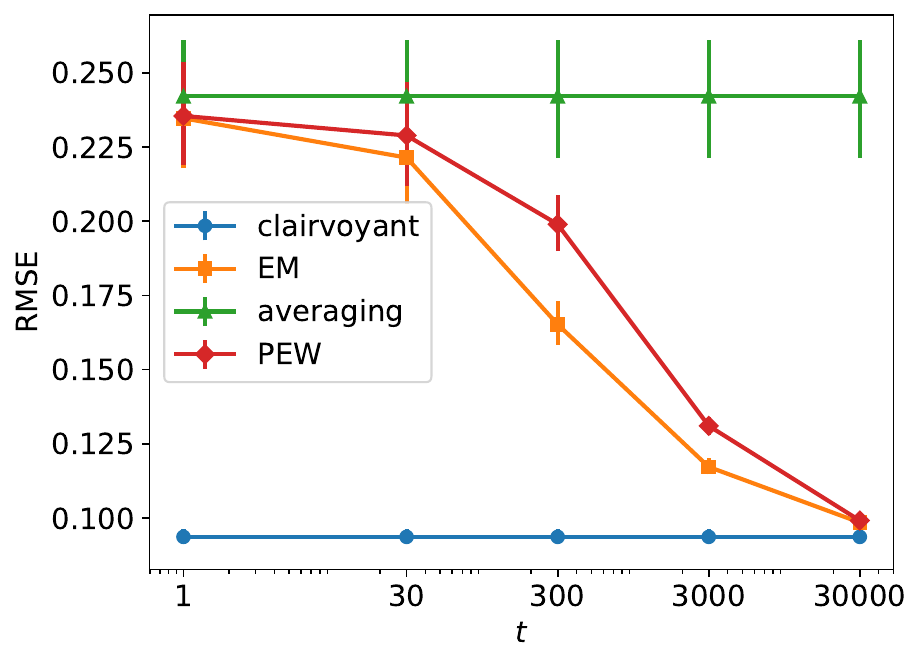}
        \subcaption{$K=30$}
        \label{fig:k_30_large}
    \end{minipage}

  \caption{Performance comparison of different policies for different values of $K$ and $t$.  The performance metric is mean squared error (MSE), but we plot root mean squared error (RMSE). We do this because RMSE has the same unit as the outcomes and estimates, while MSE does not.
  }
  \label{fig:perf_comp}
\end{figure}

\section{Literature Review}

Our literature review is structured around two main categories of prior work.  Initially, we examine research directly related to our problem setting, where true outcomes are not observed, and a center interacts with the same crowdworkers multiple times.  We aim to succinctly outline the range of algorithms developed for this setting and delineate how these differ from our proposed method, predict-each-worker. For a more comprehensive treatment of this problem setting, see the survey papers by \cite{zhang2016learning, sheng2019machine, zhang2022survey}.  Following this, we address research pertaining to different but related settings, such as those where true outcomes could be observed or interaction with crowdworkers is limited to a single round.  This part of the review highlights that our assumptions do not hold for all problem settings of interest and recognizes methods more suited to these alternative settings.  Such a comparison is useful, as it informs the reader about the boundaries of applicability of our method and aids in assessing its suitability for different application contexts.

\textbf{Our problem setting.} One of the most popular approaches for this problem setting is to posit a probabilistic model that captures crowdworker patterns and then approximate the maximum {\it a posteriori} (MAP) estimates of the parameters that characterize this probabilistic model.  The EM algorithm is one way to approximate these MAP estimates and has been applied to a variety of problems.    \citet{li2014confidence}, \citet{kara2015modeling}, and \citet{cai2020variational} use EM to learn a noise covariance matrix for the case where crowdworker estimates are continuous whereas \citet{dawid1979maximum} and \citet{zhang2016spectral} use EM to learn a so-called confusion matrix for the case where crowdworker estimates are categorical.  Other algorithms to approximate the MAP estimates have been considered in \citep{moreno2015bayesian, bach2017learning, li2019exploiting}. These algorithms estimate correlations when crowdworker estimates are categorical. Some other works assume that linear models sufficiently capture patterns between crowdworkers conditioned on relevant contexts like images or text \citep{raykar2010learning, rodrigues2018deep, li2021crowdsourcing, chen2021structured}.  The algorithms proposed in these works can be interpreted as estimating the MAP estimates of the parameters of crowdworker models conditioned on the context.  However, unlike our approach, all the algorithms discussed here are either difficult to extend or become computationally onerous if representing complex patterns requires flexible machine learning models like neural networks.

\citet{cheng2022many} consider fitting parameters of a stylized model that characterizes skills and dependencies among crowdworkers and then using the model to improve group estimates.  They demonstrate that resulting group estimates can turn out to be worse than averaging if the model is misspecified.  This casts doubt on the practicality of approaches that aim to improve on averaging.  Predict-each-worker overcomes this limitation by enabling the use of neural networks, which behave similarly to nonparametric models and thus do not suffer from misspecification to the extent that stylized models do.

\textbf{Other problem settings.}  Some prior works consider a problem setting where some or all the outcomes could be observed.  \cite{tang2011semi}, \cite{atarashi2018semi}, and \cite{gaunt2016training} develop so-called semi-supervised techniques to improve aggregation for machine learning problems.  Another line of work \citep{budescu2015identifying, merkle2017neglected, satopaa2021bias} focuses on generating improved forecasts of consequential future events.  By having the ability to observe some or all true outcomes, these works propose algorithms that make fewer restrictive assumptions about how crowdworkers relate.  Hence, for settings where ground truth information is available for a large number of crowdsourcing tasks, methods in these prior works could be more appropriate, but if ground truth information is sparse, then our approach predict-each-worker could be more helpful.

Now, we discuss a setting where the center only gets to interact with crowdworkers for one round. To gauge the skill level in this setting, prior works \citep{prelec2017solution, palley2019extracting, palley2023boosting} let each crowdworker not only guess the true quantity but also guess what the mean of the other crowdworkers' estimates will be. The intuition is that crowdworkers who can accurately predict how others will respond will tend to be more skilled.  In problem settings where each crowdsourcing task is very different from the others, using methods in the works above might be more beneficial, but if crowdsourcing tasks are similar and allow generalization, then our method could perform better.

While we treat the mechanism for seeking crowdworker estimates as given, there are several works that try to improve the method for collecting crowdworker estimates. We believe that the suggestions given in these works should be used when applying our method. There are three bodies of work that we discuss. First, works by \citet{da2020harnessing} and \citet{frey2021social} suggest that crowdworkers should not get access to estimates of others before they generate their own estimates. These works, using theoretical models and experiments, show that if this requirement is not met, errors made by crowdworkers can covary a lot and may not average out to zero, even with many crowdworkers.  As our work relies on the usefulness of averaging, it may be necessary to blind crowdworkers to each other's estimates.

Second, most crowdworkers get paid.  Accordingly, some studies examine how monetary incentives can raise the quality of crowdworkers’ estimates.  The work by \citet{mason2009financial} observed that increasing the pay did not increase the amount of quality but only increased the number of tasks that were done. Later, the work by \citet{ho2015incentivizing} pointed out that the observation made by \citet{mason2009financial} is usually true when the tasks are easy but would not hold if tasks are ``effort responsive."   Their experimental study verified this claim.  

Third, works \citep{faltings2014incentives, eickhoff2018cognitive, geiger2020garbage} discuss how cognitive biases can negatively affect the quality of crowdsourced data. To detect and mitigate such biases, the work by \citet{draws2021checklist} proposes a checklist that designers of crowdsourcing tasks should use.

\section{Concluding Remarks}
In this work, we proposed a novel approach to crowdsourcing that we call predict-each-worker.  We showed that it performs better than sample averaging, matches the performance of EM with a large dataset, and is asymptotically optimal for a Gaussian data-generating process. Additionally, we discussed how our algorithm can accommodate complex patterns between crowdworkers using neural networks. Taken together, our algorithm's performance and flexibility make it an attractive alternative to EM and other prior work.

Our work motivates many new directions for future work.  First, our aggregation rule is limited to cases where the estimates take continuous values.  We leave it for future work to develop simple aggregation rules that can be applied when the estimates are categorical, thereby making our approach directly applicable to many machine-learning problems like supervised learning \citep{zhang2016learning} and fine-tuning language models \citep{ouyang2022training}.  It would also be interesting to extend our approach to handle cases where estimates take complex forms, like in the single-question crowdsourcing problem \citep{prelec2017solution} or when they take forms like text or bounding boxes \citep{braylan2020modeling}.  Second, for simplicity and clarity, we restricted our focus to simulated data.  We leave it for future work to conduct a more thorough experimental evaluation using real-world datasets. 

\section*{Acknowledements}

This work was partly inspired by conversations with Morteza Ibrahimi and John Maggs. The paper also benefited from discussions with Saurabh Kumar and Wanqiao Xu. We gratefully acknowledge the Robert Bosch Stanford Graduate Fellowship, the NSF GRFP Fellowship, and financial support from the Army Research Office through grant W911NF2010055 and SCBX through the Stanford Institute for Human-Centered Artificial Intelligence.

\bibliography{references}

\begin{thebibliography}{59}
\providecommand{\natexlab}[1]{#1}
\providecommand{\url}[1]{\texttt{#1}}
\expandafter\ifx\csname urlstyle\endcsname\relax
  \providecommand{\doi}[1]{doi: #1}\else
  \providecommand{\doi}{doi: \begingroup \urlstyle{rm}\Url}\fi

\bibitem[Anthropic(2023)]{anthropic2023introducing}
Anthropic.
\newblock Introducing claude.
\newblock \url{https://www.anthropic.com/index/introducing-claude}, 2023.

\bibitem[Atarashi et~al.(2018)Atarashi, Oyama, and Kurihara]{atarashi2018semi}
Atarashi, K., Oyama, S., and Kurihara, M.
\newblock Semi-supervised learning from crowds using deep generative models.
\newblock In \emph{Proceedings of the AAAI Conference on Artificial
  Intelligence}, volume~32, 2018.

\bibitem[Bach et~al.(2017)Bach, He, Ratner, and R{\'e}]{bach2017learning}
Bach, S.~H., He, B., Ratner, A., and R{\'e}, C.
\newblock Learning the structure of generative models without labeled data.
\newblock In \emph{International Conference on Machine Learning}, pages
  273--282. PMLR, 2017.

\bibitem[Boussioux et~al.(2023)Boussioux, Lane, Zhang, Jacimovic, and
  Lakhani]{boussioux2023crowdless}
Boussioux, L., Lane, J.~N., Zhang, M., Jacimovic, V., and Lakhani, K.~R.
\newblock {T}he {C}rowdless {F}uture? {H}ow {G}enerative {AI} is {S}haping the
  {F}uture of {H}uman {C}rowdsourcing.
\newblock Technical Report 24-005, Harvard Business School Technology \&
  Operations Management Unit, 2023.
\newblock URL \url{https://ssrn.com/abstract=4533642}.

\bibitem[Braylan and Lease(2020)]{braylan2020modeling}
Braylan, A. and Lease, M.
\newblock Modeling and aggregation of complex annotations via annotation
  distances.
\newblock In \emph{Proceedings of The Web Conference 2020}, pages 1807--1818,
  2020.

\bibitem[Budescu and Chen(2015)]{budescu2015identifying}
Budescu, D.~V. and Chen, E.
\newblock Identifying expertise to extract the wisdom of crowds.
\newblock \emph{Management science}, 61\penalty0 (2):\penalty0 267--280, 2015.

\bibitem[Cai et~al.(2020)Cai, Nguyen, Lim, and Wynter]{cai2020variational}
Cai, D., Nguyen, D.~T., Lim, S.~H., and Wynter, L.
\newblock Variational bayesian inference for crowdsourcing predictions.
\newblock In \emph{2020 59th IEEE Conference on Decision and Control (CDC)},
  pages 3166--3172. IEEE, 2020.

\bibitem[Chen et~al.(2021)Chen, Wang, Sun, Chen, Han, Liu, and
  Yang]{chen2021structured}
Chen, Z., Wang, H., Sun, H., Chen, P., Han, T., Liu, X., and Yang, J.
\newblock Structured probabilistic end-to-end learning from crowds.
\newblock In \emph{Proceedings of the Twenty-Ninth International Conference on
  International Joint Conferences on Artificial Intelligence}, pages
  1512--1518, 2021.

\bibitem[Cheng et~al.(2022)Cheng, Asi, and Duchi]{cheng2022many}
Cheng, C., Asi, H., and Duchi, J.
\newblock How many labelers do you have? {A} closer look at gold-standard
  labels.
\newblock \emph{arXiv preprint arXiv:2206.12041}, 2022.

\bibitem[Cosma and Evers(2010)]{cosma2010markov}
Cosma, I.~A. and Evers, L.
\newblock Markov chains and {Monte Carlo} methods.
\newblock \emph{African Institute for Mathematical Sciences}, 2010.

\bibitem[Da and Huang(2020)]{da2020harnessing}
Da, Z. and Huang, X.
\newblock Harnessing the wisdom of crowds.
\newblock \emph{Management Science}, 66\penalty0 (5):\penalty0 1847--1867,
  2020.

\bibitem[Dalkey and Helmer(1963)]{dalkey1963experimental}
Dalkey, N. and Helmer, O.
\newblock An experimental application of the {D}elphi method to the use of
  experts.
\newblock \emph{Management science}, 9\penalty0 (3):\penalty0 458--467, 1963.

\bibitem[Dawid and Skene(1979)]{dawid1979maximum}
Dawid, A.~P. and Skene, A.~M.
\newblock Maximum likelihood estimation of observer error-rates using the em
  algorithm.
\newblock \emph{Journal of the Royal Statistical Society: Series C (Applied
  Statistics)}, 28\penalty0 (1):\penalty0 20--28, 1979.

\bibitem[de~Meyrick(2003)]{de2003delphi}
de~Meyrick, J.
\newblock The {D}elphi method and health research.
\newblock \emph{Health education}, 103\penalty0 (1):\penalty0 7--16, 2003.

\bibitem[Draws et~al.(2021)Draws, Rieger, Inel, Gadiraju, and
  Tintarev]{draws2021checklist}
Draws, T., Rieger, A., Inel, O., Gadiraju, U., and Tintarev, N.
\newblock A checklist to combat cognitive biases in crowdsourcing.
\newblock In \emph{Proceedings of the AAAI conference on human computation and
  crowdsourcing}, volume~9, pages 48--59, 2021.

\bibitem[Eickhoff(2018)]{eickhoff2018cognitive}
Eickhoff, C.
\newblock Cognitive biases in crowdsourcing.
\newblock In \emph{Proceedings of the eleventh ACM international conference on
  web search and data mining}, pages 162--170, 2018.

\bibitem[Faltings et~al.(2014)Faltings, Jurca, Pu, and
  Tran]{faltings2014incentives}
Faltings, B., Jurca, R., Pu, P., and Tran, B.~D.
\newblock Incentives to counter bias in human computation.
\newblock In \emph{Proceedings of the AAAI Conference on Human Computation and
  Crowdsourcing}, volume~2, pages 59--66, 2014.

\bibitem[Frey and Van~de Rijt(2021)]{frey2021social}
Frey, V. and Van~de Rijt, A.
\newblock Social influence undermines the wisdom of the crowd in sequential
  decision making.
\newblock \emph{Management science}, 67\penalty0 (7):\penalty0 4273--4286,
  2021.

\bibitem[Gallier(2010)]{gallier2010notes}
Gallier, J.~H.
\newblock Notes on the {S}chur complement.
\newblock 2010.

\bibitem[Gaunt et~al.(2016)Gaunt, Borsa, and Bachrach]{gaunt2016training}
Gaunt, A., Borsa, D., and Bachrach, Y.
\newblock Training deep neural nets to aggregate crowdsourced responses.
\newblock In \emph{Proceedings of the Thirty-Second Conference on Uncertainty
  in Artificial Intelligence. AUAI Press}, volume 242251, 2016.

\bibitem[Geiger et~al.(2020)Geiger, Yu, Yang, Dai, Qiu, Tang, and
  Huang]{geiger2020garbage}
Geiger, R.~S., Yu, K., Yang, Y., Dai, M., Qiu, J., Tang, R., and Huang, J.
\newblock Garbage in, garbage out? {D}o machine learning application papers in
  social computing report where human-labeled training data comes from?
\newblock In \emph{Proceedings of the 2020 Conference on Fairness,
  Accountability, and Transparency}, pages 325--336, 2020.

\bibitem[Google(2023)]{google2023bard}
Google.
\newblock Bard, 2023.
\newblock \url{https://bard.google.com/}, 2023.

\bibitem[Ho et~al.(2015)Ho, Slivkins, Suri, and Vaughan]{ho2015incentivizing}
Ho, C.-J., Slivkins, A., Suri, S., and Vaughan, J.~W.
\newblock Incentivizing high quality crowdwork.
\newblock In \emph{Proceedings of the 24th International Conference on World
  Wide Web}, pages 419--429, 2015.

\bibitem[Kara et~al.(2015)Kara, Genc, Aran, and Akarun]{kara2015modeling}
Kara, Y.~E., Genc, G., Aran, O., and Akarun, L.
\newblock Modeling annotator behaviors for crowd labeling.
\newblock \emph{Neurocomputing}, 160:\penalty0 141--156, 2015.

\bibitem[Kenter and De~Rijke(2015)]{kenter2015short}
Kenter, T. and De~Rijke, M.
\newblock Short text similarity with word embeddings.
\newblock In \emph{Proceedings of the 24th ACM international on conference on
  information and knowledge management}, pages 1411--1420, 2015.

\bibitem[Kwan(2014)]{kwan2014regression}
Kwan, C.~C.
\newblock A regression-based interpretation of the inverse of the sample
  covariance matrix.
\newblock \emph{Spreadsheets in Education}, 7\penalty0 (1), 2014.

\bibitem[Larrick and Soll(2006)]{larrick2006intuitions}
Larrick, R.~P. and Soll, J.~B.
\newblock Intuitions about combining opinions: {M}isappreciation of the
  averaging principle.
\newblock \emph{Management science}, 52\penalty0 (1):\penalty0 111--127, 2006.

\bibitem[Li et~al.(2014)Li, Li, Gao, Su, Zhao, Demirbas, Fan, and
  Han]{li2014confidence}
Li, Q., Li, Y., Gao, J., Su, L., Zhao, B., Demirbas, M., Fan, W., and Han, J.
\newblock A confidence-aware approach for truth discovery on long-tail data.
\newblock \emph{Proceedings of the VLDB Endowment}, 8\penalty0 (4):\penalty0
  425--436, 2014.

\bibitem[Li et~al.(2021)Li, Huang, and Chen]{li2021crowdsourcing}
Li, S.-Y., Huang, S.-J., and Chen, S.
\newblock Crowdsourcing aggregation with deep {B}ayesian learning.
\newblock \emph{Science China Information Sciences}, 64:\penalty0 1--11, 2021.

\bibitem[Li et~al.(2019)Li, Rubinstein, and Cohn]{li2019exploiting}
Li, Y., Rubinstein, B., and Cohn, T.
\newblock Exploiting worker correlation for label aggregation in crowdsourcing.
\newblock In \emph{International conference on machine learning}, pages
  3886--3895. PMLR, 2019.

\bibitem[Liu et~al.(2021)Liu, Zhang, Hou, Mian, Wang, Zhang, and
  Tang]{liu2021self}
Liu, X., Zhang, F., Hou, Z., Mian, L., Wang, Z., Zhang, J., and Tang, J.
\newblock Self-supervised {L}earning: {G}enerative or {C}ontrastive.
\newblock \emph{IEEE transactions on knowledge and data engineering},
  35\penalty0 (1):\penalty0 857--876, 2021.

\bibitem[Mars(2022)]{mars2022word}
Mars, M.
\newblock From word embeddings to pre-trained language models: {A}
  state-of-the-art walkthrough.
\newblock \emph{Applied Sciences}, 12\penalty0 (17):\penalty0 8805, 2022.

\bibitem[Mason and Watts(2009)]{mason2009financial}
Mason, W. and Watts, D.~J.
\newblock Financial incentives and the "performance of crowds".
\newblock In \emph{Proceedings of the ACM SIGKDD workshop on human
  computation}, pages 77--85, 2009.

\bibitem[Merkle et~al.(2017)Merkle, Steyvers, Mellers, and
  Tetlock]{merkle2017neglected}
Merkle, E.~C., Steyvers, M., Mellers, B., and Tetlock, P.~E.
\newblock A neglected dimension of good forecasting judgment: The questions we
  choose also matter.
\newblock \emph{International Journal of Forecasting}, 33\penalty0
  (4):\penalty0 817--832, 2017.

\bibitem[Moreno et~al.(2015)Moreno, Art{\'e}s-Rodr{\'\i}guez, Teh, and
  Perez-Cruz]{moreno2015bayesian}
Moreno, P.~G., Art{\'e}s-Rodr{\'\i}guez, A., Teh, Y.~W., and Perez-Cruz, F.
\newblock Bayesian nonparametric crowdsourcing.
\newblock \emph{Journal of Machine Learning Research}, 16\penalty0
  (48):\penalty0 1607--1627, 2015.

\bibitem[{National Weather Service}(2024)]{NWS2024Timeline}
{National Weather Service}.
\newblock Timeline of the national weather service.
\newblock \url{https://www.weather.gov/timeline}, 2024.
\newblock Accessed: 2024-01-28.

\bibitem[Ollion et~al.(2023)Ollion, Shen, Macanovic, and
  Chatelain]{ollion2023chatgpt}
Ollion, E., Shen, R., Macanovic, A., and Chatelain, A.
\newblock Chat{GPT} for {T}ext {A}nnotation? {M}ind the {H}ype!
\newblock 2023.

\bibitem[OpenAI(2023)]{openai2023gpt}
OpenAI.
\newblock {GPT}-4 technical report.
\newblock \emph{arXiv preprint arXiv:2303.08774}, 2:\penalty0 13, 2023.

\bibitem[Ouyang et~al.(2022)Ouyang, Wu, Jiang, Almeida, Wainwright, Mishkin,
  Zhang, Agarwal, Slama, Ray, et~al.]{ouyang2022training}
Ouyang, L., Wu, J., Jiang, X., Almeida, D., Wainwright, C., Mishkin, P., Zhang,
  C., Agarwal, S., Slama, K., Ray, A., et~al.
\newblock Training language models to follow instructions with human feedback.
\newblock \emph{Advances in Neural Information Processing Systems},
  35:\penalty0 27730--27744, 2022.

\bibitem[Palley and Satop{\"a}{\"a}(2023)]{palley2023boosting}
Palley, A.~B. and Satop{\"a}{\"a}, V.~A.
\newblock Boosting the wisdom of crowds within a single judgment problem:
  Weighted averaging based on peer predictions.
\newblock \emph{Management Science}, 2023.

\bibitem[Palley and Soll(2019)]{palley2019extracting}
Palley, A.~B. and Soll, J.~B.
\newblock Extracting the wisdom of crowds when information is shared.
\newblock \emph{Management Science}, 65\penalty0 (5):\penalty0 2291--2309,
  2019.

\bibitem[Prelec et~al.(2017)Prelec, Seung, and McCoy]{prelec2017solution}
Prelec, D., Seung, H.~S., and McCoy, J.
\newblock A solution to the single-question crowd wisdom problem.
\newblock \emph{Nature}, 541\penalty0 (7638):\penalty0 532--535, 2017.

\bibitem[Raykar et~al.(2010)Raykar, Yu, Zhao, Valadez, Florin, Bogoni, and
  Moy]{raykar2010learning}
Raykar, V.~C., Yu, S., Zhao, L.~H., Valadez, G.~H., Florin, C., Bogoni, L., and
  Moy, L.
\newblock Learning from crowds.
\newblock \emph{Journal of machine learning research}, 11\penalty0 (4), 2010.

\bibitem[Rodrigues and Pereira(2018)]{rodrigues2018deep}
Rodrigues, F. and Pereira, F.
\newblock Deep learning from crowds.
\newblock In \emph{Proceedings of the AAAI conference on artificial
  intelligence}, volume~32, 2018.

\bibitem[Satop{\"a}{\"a} et~al.(2021)Satop{\"a}{\"a}, Salikhov, Tetlock, and
  Mellers]{satopaa2021bias}
Satop{\"a}{\"a}, V.~A., Salikhov, M., Tetlock, P.~E., and Mellers, B.
\newblock Bias, information, noise: The bin model of forecasting.
\newblock \emph{Management Science}, 67\penalty0 (12):\penalty0 7599--7618,
  2021.

\bibitem[Sheng and Zhang(2019)]{sheng2019machine}
Sheng, V.~S. and Zhang, J.
\newblock Machine learning with crowdsourcing: A brief summary of the past
  research and future directions.
\newblock In \emph{Proceedings of the AAAI conference on artificial
  intelligence}, volume~33, pages 9837--9843, 2019.

\bibitem[Skulmoski et~al.(2007)Skulmoski, Hartman, and
  Krahn]{skulmoski2007delphi}
Skulmoski, G.~J., Hartman, F.~T., and Krahn, J.
\newblock The {D}elphi method for graduate research.
\newblock \emph{Journal of Information Technology Education: Research},
  6\penalty0 (1):\penalty0 1--21, 2007.

\bibitem[{Smithsonian Institution
  Archives}(2024)]{Smithsonian2024HenryMeteorology}
{Smithsonian Institution Archives}.
\newblock Meteorology: Joseph henry's forecasts.
\newblock
  \url{https://siarchives.si.edu/history/featured-topics/henry/meteorology},
  2024.
\newblock Accessed: 2024-01-28.

\bibitem[Tang and Lease(2011)]{tang2011semi}
Tang, W. and Lease, M.
\newblock Semi-supervised consensus labeling for crowdsourcing.
\newblock In \emph{SIGIR 2011 workshop on crowdsourcing for information
  retrieval (CIR)}, pages 1--6, 2011.

\bibitem[Tetlock and Gardner(2016)]{tetlock2016superforecasting}
Tetlock, P.~E. and Gardner, D.
\newblock \emph{Superforecasting: The art and science of prediction}.
\newblock Random House, 2016.

\bibitem[T{\"o}rnberg(2023)]{tornberg2023chatgpt}
T{\"o}rnberg, P.
\newblock Chat{GPT}-4 outperforms experts and crowd workers in annotating
  political twitter messages with zero-shot learning.
\newblock \emph{arXiv preprint arXiv:2304.06588}, 2023.

\bibitem[Touvron et~al.(2023)Touvron, Martin, Stone, Albert, Almahairi, Babaei,
  Bashlykov, Batra, Bhargava, Bhosale, et~al.]{touvron2023llama}
Touvron, H., Martin, L., Stone, K., Albert, P., Almahairi, A., Babaei, Y.,
  Bashlykov, N., Batra, S., Bhargava, P., Bhosale, S., et~al.
\newblock Llama 2: Open foundation and fine-tuned chat models.
\newblock \emph{arXiv preprint arXiv:2307.09288}, 2023.

\bibitem[Vaughan(2017)]{vaughan2017making}
Vaughan, J.~W.
\newblock Making better use of the crowd: How crowdsourcing can advance machine
  learning research.
\newblock \emph{J. Mach. Learn. Res.}, 18\penalty0 (1):\penalty0 7026--7071,
  2017.

\bibitem[Yuan(2010)]{yuan2010high}
Yuan, M.
\newblock High dimensional inverse covariance matrix estimation via linear
  programming.
\newblock \emph{The Journal of Machine Learning Research}, 11:\penalty0
  2261--2286, 2010.

\bibitem[Zhang et~al.(2022)Zhang, Hsieh, Yu, Zhang, and
  Ratner]{zhang2022survey}
Zhang, J., Hsieh, C.-Y., Yu, Y., Zhang, C., and Ratner, A.
\newblock A survey on programmatic weak supervision.
\newblock \emph{arXiv preprint arXiv:2202.05433}, 2022.

\bibitem[Zhang et~al.(2016{\natexlab{a}})Zhang, Wu, and
  Sheng]{zhang2016learning}
Zhang, J., Wu, X., and Sheng, V.~S.
\newblock Learning from crowdsourced labeled data: a survey.
\newblock \emph{Artificial Intelligence Review}, 46:\penalty0 543--576,
  2016{\natexlab{a}}.

\bibitem[Zhang et~al.(2016{\natexlab{b}})Zhang, Chen, Zhou, and
  Jordan]{zhang2016spectral}
Zhang, Y., Chen, X., Zhou, D., and Jordan, M.~I.
\newblock Spectral methods meet {EM}: {A} provably optimal algorithm for
  crowdsourcing.
\newblock \emph{The Journal of Machine Learning Research}, 17\penalty0
  (1):\penalty0 3537--3580, 2016{\natexlab{b}}.

\bibitem[Zheng et~al.(2017)Zheng, Li, Li, Shan, and Cheng]{zheng2017truth}
Zheng, Y., Li, G., Li, Y., Shan, C., and Cheng, R.
\newblock Truth inference in crowdsourcing: Is the problem solved?
\newblock \emph{Proceedings of the VLDB Endowment}, 10\penalty0 (5):\penalty0
  541--552, 2017.

\bibitem[Zhu et~al.(2023)Zhu, Zhang, Haq, Hui, and Tyson]{zhu2023can}
Zhu, Y., Zhang, P., Haq, E.-U., Hui, P., and Tyson, G.
\newblock Can {C}hat{GPT} reproduce human-generated labels? {A} study of social
  computing tasks.
\newblock \emph{arXiv preprint arXiv:2304.10145}, 2023.

\end{thebibliography}

\appendix

\section{Proof of Theorem~\ref{thm:ssl_sufficiency}}
\label{app:proof_ssl_sufficiency}

To prove Theorem~\ref{thm:ssl_sufficiency}, we first state the Hammersley-Clifford Theorem.  It basically says that under minor technical conditions, ``leave-one-out" conditional distributions specify the joint distribution.
\begin{lemma}[Hammersley-Clifford Theorem]
\label{lem:full_cond_joint_dist}
Let $X \in \Re^d$ be a random variable.  Suppose its probability density function ${\bf p}_{X}$ has product-form support. Then, for all $\xi \in {\rm support}({\bf p}_{X})$ and $x \in {\rm support}({\bf p}_{X})$,
\begin{align*}
        {\bf p}_X(x) \propto \prod_{i=1}^d \frac{{\bf p}_{X_i | X_{-i}}(x_i | x_1,\cdots,x_{i-1},\xi_{i+1},\cdots,\xi_{d})}{{\bf p}_{X_i | X_{-i}}(\xi_i | x_1,\cdots,x_{i-1},\xi_{i+1},\cdots,\xi_{d})}
    \end{align*}
\end{lemma}
The proof is given below. It is adapted from \cite{cosma2010markov}.
\begin{proof}
    \begin{align*}
        {\bf p}_X(x_1, \cdots, x_d) \overset{(a)}{=} &{\bf p}_{X_{-d}}(x_{1},\cdots,x_{d-1}) {\bf p}_{X_d | X_{-d}}(x_d | x_{-d}) \\
        \overset{(b)}{=}  & \frac{{\bf p}_X(x_1, \cdots, x_{d-1}, \xi_d)}{{\bf p}_{X_d | X_{-d}}(\xi_d | x_{-d})}{\bf p}_{X_d | X_{-d}}(x_d | x_{-d}) \\
        = & {\bf p}_X(x_1, \cdots, x_{d-1}, \xi_d)\frac{{\bf p}_{X_d | X_{-d}}(x_d | x_{-d})}{{\bf p}_{X_d | X_{-d}}(\xi_d | x_{-d})} \\
        = &\cdots \\
        = &{\bf p}_X(\xi_1, \cdots, \xi_d) \prod_{i=1}^d \frac{{\bf p}_{X_i | X_{-i}}(x_i | x_1,\cdots,x_{i-1},\xi_{i+1},\cdots,\xi_{d})}{{\bf p}_{X_i | X_{-i}}(\xi_i | x_1,\cdots,x_{i-1},\xi_{i+1},\cdots,\xi_{d})}.
    \end{align*}
    Here, $(a)$ and $(b)$ follow from Bayes rule.  The product-form support ensures that for all $i$, ${\bf p}_{X_i|X_{-i}}(x_i|x_{1:i-1},\xi_{i+1:d})$ and ${\bf p}_{X_i | X_{-i}}(\xi_i | x_{1:i-1},\xi_{i+1:d})$ are both well-defined and non-zero.
\end{proof}

Now, we will prove Theorem~\ref{thm:ssl_sufficiency}. We restate the theorem statement for clarity.
\SSLSufficiency*
\begin{proof}
As $\Pr(Y_{t} \in \cdot | \theta)$ is absolutely continuous, and the corresponding density has product-form support, Lemma~\ref{lem:full_cond_joint_dist} implies that $\{P_*^{(k)}\}_{k=1}^K$ determine $\Pr(Y_{t} \in \cdot | \theta)$. Moreover, as $\theta$ is minimal, in the sense that, $\Pr(Y_{t} \in \cdot | \theta)$ determines $\theta$, we have that $\{P_*^{(k)}\}_{k=1}^K$ determines $\theta$.

\end{proof}

\section{Proof of Theorem~\ref{thm:opt_wts_opt_ssl_params}}
\label{app:opt_wts_opt_ssl_params}

The proof is based on the following lemma, a well-known result in the literature for covariance matrix estimation \citep{yuan2010high, kwan2014regression}. It shows how a covariance matrix can be expressed in terms of coefficients and errors of certain linear regression problems.
\begin{restatable}{lemma}{covLinearReg}
   \label{lem:cov_linear_reg}
    Consider a covariance matrix $S \in \mathcal{S}^K_{++}$. Suppose $A \sim \normal(0, S)$. For each $k \in \{1,\cdots,K \}$, let
    \begin{align*}
        u^{(k)} &\in \argmin_{u \in \Re^{K-1}} \E[(A_k - u^\top A_{-k})^2] \\
        \ell^{(k)} &= \min_{u \in \Re^{K-1}} \E[(A_k - u^\top A_{-k})^2]
    \end{align*}
    Then,
    \begin{align*}
        S^{-1}_{k,k'} = 
        \begin{cases}
            \frac{1}{\ell^{(k)}} & \text{if } k=k' \\
            -\frac{u^{(k,k')}}{\ell^{(k)}} & \text{if } k \neq k'. 
        \end{cases}
    \end{align*} 
\end{restatable}
\noindent We defer the proof to Appendix~\ref{app:cov_linear_reg_proof}.

Now, we prove Theorem~\ref{thm:opt_wts_opt_ssl_params}. We restate the statement for clarity.
\optWtsOptSSLParams*
\begin{proof}
    As for any $t$, $Y_t | S_* \sim \normal(0, S_*)$, we have
    \begin{align*}
        \indiWtOpt &\in \argmin_{u \in \Re^{K-1}} \E[(Y_{t,k} - u^\top Y_{t,-k})^2 | S_*]; \\
        \indiLossOpt &= \min_{u \in \Re^{K-1}} \E[(Y_{t,k} - u^\top Y_{t,-k})^2 | S_*].
    \end{align*}
    Next, by Lemma~\ref{lem:cov_linear_reg},
    \begin{align*}
        S^{-1}_{*,k,k'} = 
        \begin{cases}
            \frac{1}{\indiLossOpt} & \text{if } k=k' \\
            -\frac{u_*^{(k,k')}}{\indiLossOpt} & \text{if } k \neq k'. 
        \end{cases}
    \end{align*}
    Second, recall that $\nu_* = \overline{v} S_*^{-1} \1$. By substituting $S^{-1}_*$ from above in $\overline{v} S_*^{-1} \1$, we get the result.
\end{proof}

\subsection{Proof of Lemma~\ref{lem:cov_linear_reg}}
\label{app:cov_linear_reg_proof}
To prove Lemma~\ref{lem:cov_linear_reg}, we state another lemma that allows us to write the inverse of a block matrix in terms of its blocks.
\begin{lemma}
\label{lem:block_inv}
Let $\mathfrak{M}$ be an invertible matrix. Suppose it has the following block form:
\begin{align*}
    \mathfrak{M} = \begin{pmatrix}
                    A & B \\
                    C & D
                    \end{pmatrix}.
\end{align*}
If $D$ is invertible, then
\[
\mathfrak{M}^{-1}
=
\begin{pmatrix}
(A - BD^{-1}C)^{-1} & -(A - BD^{-1}C)^{-1}BD^{-1} \\
-D^{-1}C(A - BD^{-1}C)^{-1} & D^{-1} + D^{-1}C(A - BD^{-1}C)^{-1}BD^{-1}
\end{pmatrix}
\]
\end{lemma}
Please refer to pages 1 and 2 of \citep{gallier2010notes} on Schur complements for a proof of this lemma.

Now, we prove Lemma~\ref{lem:cov_linear_reg}. We restate it for clarity.
\covLinearReg*

\begin{proof}

To prove this Lemma, we begin by writing $S$ as follows.
\begin{align*}
    S = 
    \begin{bmatrix}
        S_{1,1} & S_{-1,1} \\
        S_{1,-1} & S_{-1,-1}
    \end{bmatrix}.
\end{align*}

As $S$ is positive definite, it is easy to show that $u^{(1)}$ is unique. This unique value and the corresponding error $\ell^{(1)}$ can be expressed in terms of elements of $S$. Specifically, we have
\begin{align*}
    u^{(1)} &= S_{-1,-1}^{-1}S_{-1,1}; \\
    \ell^{(1)} &= S_{1,1} - S_{1,-1}S_{-1,-1}^{-1}S_{-1,1}.
\end{align*}

Then, we apply Lemma~\ref{lem:block_inv} on matrix $S_*$ to get that the first row of $S^{-1}$ is the following.
\begin{align*}
    S^{-1}_{1,:} = \left[\frac{1}{\ell^{(1)}}, ~- \frac{(u^{(1)})^\top}{\ell^{(1)}} \right]
\end{align*}
In other words, for $k=1$, 
\begin{align*}
    S^{-1}_{k,k'} = 
    \begin{cases}
        \frac{1}{\ell^{(k)}} & \text{if } k=k' \\
        -\frac{u^{(k,k')}}{\ell^{(k)}} & \text{if } k \neq k'. 
    \end{cases}.
\end{align*} 
With some simple algebra, it is straightforward to extend this argument for other values of $k$.

\end{proof}

\optWtsOptSSLParams*

\section{Proof of Theorem~\ref{thm:consistency_thm}}
\label{app:consistency_thm}

To prove Theorem~\ref{thm:consistency_thm}, we use two results. First, Theorem~\ref{thm:opt_wts_opt_ssl_params}, a theorem that offers a characterization of $\nu_*$ in terms of clairvoyant SSL parameters.  Second, a result that shows that Bayesian linear regression is consistent for our problem. 

The following lemma formalizes our claim that Bayesian linear regression converges for our problem.
\begin{lemma}
\label{lem:blr_converges}
If $\Lambda \succ 0$, $\indiLossPrior>0$, and $\regVar \geq 0$, then for each $k \in \{1,\cdots,K \}$, $\lim_{t \to \infty} \indiWt \overset{{\rm a.s.}}{=} \indiWtOpt$ and $\lim_{t \to \infty} \indiLoss \overset{{\rm a.s.}}{=} \indiLossOpt$.
\end{lemma}
The proof is deferred to Appendix~\ref{app:blr_converges_proof}.  We assume $\Lambda \succ 0$ and $\indiLossPrior>0$ to make the proof simpler, but it is possible to prove a similar result by only assuming $\Lambda \succeq 0$ and $\indiLossPrior \geq 0$.

Now, we will prove Theorem~\ref{thm:consistency_thm}. For clarity, we restate this theorem.
\ConsistencyThm*
\begin{proof}

As $\E[Z_t | Y_{1:\infty}] = \nu_*^\top Y_t$, to prove the theorem, it suffices to prove that $\lim_{t \to \infty} \hat{\nu}_{t,k} = \nu_*$ almost surely. 

Before we prove this, we show that $\lim_{t \to \infty} \Tilde{\nu}_{t,k}$ almost surely.
\begin{align*}
    \lim_{t \to \infty} \Tilde{\nu}_{t,k} \overset{(a)}{=} &\lim_{t \to \infty} \overline{v} \frac{1 - \1^\top \indiWt}{\indiLoss} \\
    \overset{(b)}{=} &\overline{v}\frac{1 - \1^\top \indiWtOpt}{\indiLossOpt} \quad {\rm a.s.} \\
  \overset{(c)}{=} & \nu_{*,k}. \quad {\rm a.s.}
\end{align*}
$(a)$ uses the definition of $\Tilde{\nu}_{t,k}$. $(b)$ is true because of Lemma~\ref{lem:blr_converges}. $(c)$ is true because of Theorem~\ref{thm:opt_wts_opt_ssl_params}.

Now, we show $\lim_{t \to \infty} \hat{\nu}_{t,k} = \nu_*$.
\begin{align*}
  \lim_{t \to \infty} \hat{\nu}_{t,k} \overset{(a)}{=} &\lim_{t \to \infty} \left(\frac{b}{b+t-1} \Tilde{\nu}_{1,k} + \left(1- \frac{b}{b+t-1} \right) \Tilde{\nu}_{t,k} \right) \\
  \overset{(b)}{=} &\lim_{t \to \infty} \frac{b}{b+t-1} \Tilde{\nu}_{1,k} + \lim_{t \to \infty} \left(1- \frac{b}{b+t-1} \right) \Tilde{\nu}_{t,k}  \quad {\rm a.s.} \\
  \overset{(c)}{=} & 0 + \lim_{t \to \infty} \Tilde{\nu}_{t,k}  \quad {\rm a.s.} \\
  \overset{(d)}{=} & \nu_{*,k}.   \quad {\rm a.s.}
\end{align*}
$(a)$ is true using the definition of $\hat{\nu}_{t,k}$. $(b)$ is true because the first limit exists surely and because the second limit exists almost surely. $(c)$ is true $\lim_{t \to \infty}\left(1- \frac{b}{b+t-1} \right)=1$ and because $\lim_{t \to \infty} \Tilde{\nu}_{t,k}$ exists almost surely. $(d)$ uses consistency of $\Tilde{\nu}_{t,k}$.

\end{proof}

\subsection{Proof of Lemma~\ref{lem:blr_converges}}
\label{app:blr_converges_proof}

\begin{proof}

Recall that $(\indiWt, \indiLoss)$ is a minimizer of the following function, where the minimization happens over values $(u,\ell) \in \Re^{K-1} \times \Re_{+}$.
\begin{align*}
    f^{(k)}_{Y_{1:t-1},\regWt,\regVar,\indiWtPrior,\indiLossPrior}(u, \ell) 
    = &\underbrace{\frac{1}{2\ell} \left(u - \indiWtPrior \1  \right)^\top \Lambda \left(u - \indiWtPrior \1 \right) + \log \left( \left|l \Lambda^{-1} \right| \right)}_{\textrm{Normal prior}, ~\normal(\indiWtPrior \1, \ell \Lambda^{-1})} + \underbrace{\left(\frac{\regVar}{2} + 1\right)\log \ell + \frac{(\regVar + K + 1) \indiLossPrior }{2 \ell}}_{\textrm{Inverse-Gamma prior}} \\
    &\underbrace{+ \sum_{\tau < t} \frac{(Y_{\tau, k} - u^\top Y_{\tau, -k})^2}{2\ell} + \frac{(t-1)}{2} \log(\ell)}_{{\rm log\ likelihood}} + \mathfrak{c},
\end{align*}
where $\mathfrak{c}$ is an additive constant that does not depend on $u$ and $\ell$.

We claim that the following pair is the unique minimizer of $f^{(k)}_{Y_{1:t-1},\regWt,\regVar,\indiWtPrior,\indiLossPrior}$ for all $t$ if $\Lambda \succ 0$ and $\indiLossPrior>0$.
\begin{align*}
    \indiWt = &\left( \Lambda + \sum_{\tau<t} Y_{\tau,-k}Y_{\tau,-k}^\top \right)^{-1} \left( \Lambda \1 \overline{u} + \sum_{\tau<t} Y_{\tau,k} Y_{\tau,-k}\right) \\
    \indiLoss = &\frac{\regVar + K + 1}{\regVar + K + t} \overline{\ell}^{(k)} + \frac{1}{\regVar + K + t} (\indiWt - \overline{u} \1 )^\top \Lambda (\indiWt - \overline{u} \1) \\
    &+ \frac{\sum_{\tau<t} \big(Y_{\tau,k} - (\indiWt)^\top Y_{\tau,-k}\big)^2}{\regVar + K + t}
\end{align*}
To establish this uniqueness, we find the Hessian matrix for $f^{(k)}_{Y_{1:t-1},\regWt,\regVar,\indiWtPrior,\indiLossPrior}$ at $(\indiWt, \indiLoss)$ and show that it is positive definite for all $t$. This matrix is equal to 
\begin{align*}
    \begin{bmatrix}
        \frac{2}{\indiLoss} \left( \Lambda + \sum_{\tau < t} Y_{\tau,-k}Y_{\tau,-k}^\top \right) & {\bf 0}^\top \\
        {\bf 0} & \frac{\regVar + K + t }{\left(\indiLoss\right)^2}
    \end{bmatrix}.
\end{align*}
First, this matrix is well-defined because $\indiLoss>0$ for all $t$ (consequence of $\indiLossPrior$ being positive). Second, this matrix is positive-definite because $\Lambda~+~\sum_{\tau < t}Y_{\tau,-k}Y_{\tau,-k}^\top~\succ~0$ and $\regVar + K + t > 0$.

Now we will show that for each $k$, $\indiWt \overset{a.s.}{\to} \indiWtOpt$.
\begin{align*}
    \lim_{t \to \infty} \indiWt = &\lim_{t \to \infty} \left( \Lambda + \sum_{\tau<t} Y_{\tau,-k}Y_{\tau,-k}^\top \right)^{-1} \left( \Lambda \overline{u}^{(k)} + \sum_{\tau<t} Y_{\tau,k} Y_{\tau,-k}\right) \\
    = &\lim_{t \to \infty} \left( \frac{\Lambda}{t-1}+ \frac{\sum_{\tau<t} Y_{\tau,-k}Y_{\tau,-k}^\top}{t-1} \right)^{-1} \left( \frac{\Lambda \overline{u}^{(k)} }{t-1}+ \frac{\sum_{\tau<t} Y_{\tau,k} Y_{\tau,-k}}{t-1}\right) \\
    \overset{(a)}{=} &  \lim_{t \to \infty}\left( \frac{\Lambda}{t-1}+ \frac{\sum_{\tau<t} Y_{\tau,-k}Y_{\tau,-k}^\top}{t-1} \right)^{-1} \lim_{t \to \infty} \left( \frac{\Lambda \overline{u}^{(k)} }{t-1}+ \frac{\sum_{\tau<t} Y_{\tau,k} Y_{\tau,-k}}{t-1}\right) \quad {\rm a.s.} \\
    \overset{(b)}{=} &  \lim_{t \to \infty}\left( \frac{\Lambda}{t-1}+ \frac{\sum_{\tau<t} Y_{\tau,-k}Y_{\tau,-k}^\top}{t-1} \right)^{-1} S_{*,k,-k} \quad {\rm a.s.} \\
    \overset{(c)}{=} &  S_{*,-k,-k}^{-1}  S_{*,k,-k} \quad {\rm a.s.} \\
    = &  \indiWtOpt \quad {\rm a.s.}
\end{align*}
$(a)$ is a valid step because both limits exist almost surely.  $(b)$ is true because $\frac{\Lambda \overline{u}^{(k)} }{t-1} \to 0$ and law of large numbers implies that $\frac{\sum_{\tau<t} Y_{\tau,k} Y_{\tau,-k}}{t-1} \to S_{*,k,-k}$ almost surely.  $(c)$ is true because $\frac{\Lambda }{t-1} \to 0$ and law of large numbers implies that $\frac{\sum_{\tau<t} Y_{\tau,-k} Y_{\tau,-k}}{t-1} \to S_{*,-k,-k}$ almost surely.

Now we will show that for each $k$, $\indiLoss \overset{a.s.}{\to} \indiLossOpt$.
\begin{align*}
    \lim_{t \to \infty}\indiLoss = &\lim_{t \to \infty} \bigg(\frac{\regVar + K + 1}{\regVar + K + t} \overline{\ell}^{(k)} + \frac{1}{\regVar + K + t} (\indiWt - \overline{u}^{(k)})^\top \Lambda (\indiWt - \overline{u}^{(k)}) \\
    &+ \frac{\sum_{\tau<t} \big(Y_{\tau,k} - (\indiWt)^\top Y_{\tau,-k}\big)^2}{\regVar + K + t}\bigg) \\
    \overset{(a)}{=}  &\lim_{t \to \infty} \frac{\regVar + K + 1}{\regVar + K + t} \overline{\ell}^{(k)} + \lim_{t \to \infty} \frac{1}{\regVar + K + t} (\indiWt - \overline{u}^{(k)})^\top \Lambda (\indiWt - \overline{u}^{(k)}) \\
    &+ \lim_{t \to \infty} \frac{\sum_{\tau<t} \big(Y_{\tau,k} - (\indiWt)^\top Y_{\tau,-k}\big)^2}{\regVar + K + t} \quad {\rm a.s.}\\
    =  &0 + 0 + \lim_{t \to \infty} \frac{\sum_{\tau<t} \big(Y_{\tau,k} - (\indiWt)^\top Y_{\tau,-k}\big)^2}{\regVar + K + t} \quad {\rm a.s.}\\
    \overset{(b)}{=} &\lim_{t \to \infty} \frac{t-1}{\regVar + K + t} \times \bigg( \lim_{t \to \infty} \frac{\sum_{\tau<t}Y_{\tau,k}^2}{t-1} - 2 \lim_{t \to \infty} (\indiWt)^\top \times\lim_{t \to \infty} \frac{\sum_{\tau<t} Y_{\tau,k}Y_{\tau,-k} }{t-1} \\
    & + \lim_{t \to \infty} (\indiWt)^\top \times\lim_{t \to \infty} \frac{\sum_{\tau<t} Y_{\tau,-k}Y_{\tau,-k} }{t-1} \times \lim_{t \to \infty} (\indiWt) \bigg) \\
    \overset{(c)}{=} & S_{*,k,k} - 2 \left(\indiWtOpt\right)^\top S_{*,k,-k}+ \left(\indiWtOpt\right)^\top S_{*,-k,-k}\indiWtOpt \\
    \overset{(d)}{=} &S_{*,k,k} - S_{*,k,-k} S_{*,-k,-k} S_{*,-k,k} \\
    \overset{(e)}{=} &\indiLossOpt.
\end{align*}
$(a)$ is valid because each limit exists surely or almost surely.  $(b)$ is valid because each limit exists surely.  $(c)$ is true because $\lim_{t \to \infty} \frac{t-1}{\regVar + K + t} = 1$, $\indiWt \overset{a.s.}{\to} \indiWtOpt$, and $\frac{\sum_{\tau < t} Y_{\tau}Y_{\tau}^\top}{t-1} \overset{a.s.}{\to} S_*$. $(d)$ and $(e)$ follow from definitions of $\indiWtOpt$ and $\indiLossOpt$.

\end{proof}

\section{Simulation Study Details}

\subsection{Data-Generating Process Satisfies Problem-Formulation Assumptions}
\label{app:data_gen_prob_form_asmp}

Recall that we generate the data in the following manner.  For each crowdworker $k \in \mathbb{Z}_{++}$, 
\begin{align*}
\Tilde{Y}_{t,k} = \tilde{Z}_t + \sum_{n=1}^N C_{k,n} X_{t,n}.
\end{align*}
Here, $\tilde{Z}_t \sim \normal(0,1)$, $(X_t: t \in \mathbb{Z}_{++})$ is an iid sequence of standard Gaussian vectors that is independent of $\tilde{Z}_t \sim \normal(0,1)$, and $C_{k,n}$ satisfies the following. $C_{k,n} \sim \normal(0,n^{-q})$, $C_{k,n}$ is sampled independently from $\tilde{Z}_t$, the factor $X_t$, and any other coefficient $C_{k',n'}$, where $(k', n')\neq (k,n)$.

Let us verify if this data-generating process satisfies the assumptions that we made in the problem formulation. First, we will check that for any value of $t$, the elements of $(\Tilde{Y}_{t,k}:k=1,2,\cdots)$ are exchangeable.  As the elements of $(\Tilde{Y}_{t,k}:k=1,2,\cdots)$ are iid conditioned on $\Tilde{Z}_t$ and $X_t$, by de Finetti's theorem, they are exchangeable.  Second, we will check that $(\Tilde{Y}_{t}:t=1,2,\cdots)$ are exchangeable. This is true because conditioned on coefficients $C$, $(\Tilde{Y}_{t}:t=1,2,\cdots)$ are iid.  Next, we will check that $\sum_{k=1}^K \Tilde{Y}_{t,k}/K$ converges almost surely.  
\begin{align*}
    \lim_{K \to \infty} \sum_{k=1}^K \frac{\Tilde{Y}_{t,k}}{K} \overset{(a)}{=} &\lim_{K \to \infty} \frac{1}{K} \left( \tilde{Z}_t + \sum_{k=1}^K \sum_{n=1}^N C_{k,n} X_{t,n} \right) \\
    \overset{(b)}{=} & \tilde{Z}_t + \lim_{K \to \infty} \frac{1}{K}\sum_{k=1}^K \sum_{n=1}^N C_{k,n} X_{t,n} \quad {\rm a.s.} \\
    = & \tilde{Z}_t + \lim_{K \to \infty} \sum_{n=1}^N X_{t,n} \frac{1}{K}\sum_{k=1}^K  C_{k,n} \quad {\rm a.s.} \\
    \overset{(c)}{=} & \tilde{Z}_t + \sum_{n=1}^N X_{t,n} \lim_{K \to \infty} \frac{1}{K}\sum_{k=1}^K  C_{k,n} \quad {\rm a.s.} \\
    \overset{(d)}{=} & \tilde{Z}_t.
\end{align*}
Here, $(a)$ follows from the definition of $\Tilde{Y}_{t,k}$. $(b)$ follows because $\tilde{Z}_t$ exists surely, and the second term will be shown to exist almost surely. $(c)$ is valid because elements of $\{X_{t,n}\}_{n=1}^N$ are finite almost surely and because $N < \infty$. $(d)$ is true because of the strong law of large numbers.

\subsection{Hyperparameter Tuning for Predict-Each-Worker}
\label{app:jpo_tuning}
We describe how we do hyperparameter tuning for this predict-each-worker here. For each $K$, we do a separate hyperparameter search. We tune the hyperparameters for the SSL step, and then we tune the hyperparameter for the aggregation step using the best hyperparameters from the SSL step. 

The performance metric that we use to select hyperparameters in the SSL step judges the quality of the learned SSL models.  This metric is an expectation of a loss function.  The intuition behind this loss function is as follows.  Given a history $Y_{1:t-1}$, the SSL models approximate the clairvoyant joint distribution over $Y_t$.  The loss function is equal to the KL divergence between the clairvoyant distribution over $Y_t$ and the distribution implied by these SSL models (details in Appendix~\ref{app:ssl_metric}).

Empirically, when selecting hyperparameters, we noticed that there is no hyperparameter setting, that dominates across different values of $t$. Therefore, we choose the hyperparameters in two steps. First, we evaluate performance for $t \in \{ 1, K, 10 \times K, 100 \times K \}$ and filter out all hyperparameter combinations for which the performance does not monotontically improve as $t$ increases. Then, amongst the remaining hyperparameter combinations, we select the hyperparameters that result in the best performance for some large value of $t$ which we call $t_*$. We let $t_* = 100 \times K$. 

We let $\overline{u} = \frac{1}{K + 1} $ and $\indiLossPrior = 2 + \frac{2}{K + 1}$.  We let $\Lambda = \lambda ((1- \rho) I + \rho \1\1^\top)$ where $\lambda$ is the diagonal element and $\rho \lambda$ is the off-diagonal element. We search over $\rho \in \{ 0, 0.2, 0.4, 0.6, 0.8 \}$, $\lambda \in \{ \frac{0 \times K}{5}, \frac{2 \times K}{5}, \cdots, \frac{ 10 \times K}{5}  \} $ and $\regVar \in \{ 0,2,4,6 \}$.

When tuning the aggregation hyperparameter $r$, we use the best SSL parameters, and we select based on an estimate of the MSE (estimation procedure discussed in Appendix~\ref{app:mse_aw}). As with SSL parameters, we enforce that performance should monotonically improve as $t$ increases. Specifically, we evaluate for $t \in \{ 1, K, 3 \times K, 5 \times K, 7 \times K, 10 \times K \}$. Then, finally, we choose the best hyperparameters based on performance at $t_* = 10 \times K$. Note that the values $t$ we evaluate for and the value of $t_*$ is different for the SSL step and the aggregation step.

For each value of $t$, we estimate MSE by averaging across multiple training seeds. We vary $r \in \{2.5 \times K, 5 \times K, \cdots, 20 \times K \}$. The final hyperparameters can be found in Table~\ref{tab:jpo_hyperparams}.

\begin{table}
    \centering
    \begin{tabular}{|c|c|c|c|c|}
    \hline
    $\pmb{K}$ & $\pmb{\lambda}$ & $\pmb{\rho}$ & $\pmb{\regVar}$ & $\pmb{r}$ \\
    \hline 
    $10$ & $16$ & $0.4$ & $0$ & $75$ \\
    $20$ & $24$ & $0.6$ & $0$ & $150$ \\
    $30$ & $36$ & $0.6$ & $0$ & $300$ \\
    \hline 
    \end{tabular}
\caption{Final hyperparameters for different values of $K$}
\label{tab:jpo_hyperparams}
\end{table}

To tune hyperparameters for both EM and predict-each-worker, we assume access to the clairvoyant distribution $\Pr(Y_t \in \cdot | \theta)$.   We do this for simplicity.  In Appendix~\ref{app:hyperparam_tuning_real}, we discuss how hyperparameter tuning can be performed without access to $\Pr(Y_t \in \cdot | \theta)$.

\subsection{Estimating MSE}
\label{app:mse_aw}

The MSE is given by $\E[(Z_t - \hat{Z}_t^\pi)^2]$ where $\hat{Z}_t^\pi$ is the group estimate generated by a policy $\pi$ after observing $Y_{1:t}$.  Computing MSE exactly is typically computationally expensive and needs to be estimated.  Here, we discuss how we estimate MSE in the context of our experiments.  We utilize three features of our experiments to make it easier to estimate MSE.  First, we use the fact that the data is drawn from the Gaussian data-generating process.  Second, we use the knowledge of the noise covariance matrices from which the data is drawn from.  Third, we use the fact that all policies we consider generate a group estimate that is linear in $Y_t$, i.e., $\hat{Z}_t^\pi = \hat{\nu}_t^\top Y_t$ and $\hat{\nu}_t$ depends only on $Y_{1:t-1}$.

\begin{align*}
    \E[(Z_t - \hat{Z}_t^\pi)^2] \overset{(a)}{=} &\E\left[\E\left[(Z_t - \hat{Z}_t^\pi)^2 | Y_{1:t}, \Sigma_*\right]\right]; \\
    \overset{(b)}{=} &\E\left[\V[Z_t | Y_{1:t}, \Sigma_*]\right] + \E[(\E[Z_t | Y_{1:t}, \Sigma_*] - \hat{Z}_t^\pi)^2]; \\
    \overset{(c)}{=} &\E\left[\frac{1}{ \overline{v}^{-1} + \1^\top \Sigma_*^{-1}\1}\right] + \E[(\E[Z_t | Y_{1:t}, \Sigma_*] - \hat{Z}_t^\pi)^2]; \\
    \overset{(d)}{=} &\E\left[\frac{1}{ \overline{v}^{-1} + \1^\top \Sigma_*^{-1}\1}\right] + \E\left[\E[(\E[Z_t | Y_{1:t},  \Sigma_*] - \hat{Z}_t^\pi)^2 | \Sigma_*, Y_{1:t-1}]\right]; \\
    = &\E\left[\frac{1}{ \overline{v}^{-1} + \1^\top \Sigma_*^{-1}\1}\right] + \E\left[\E[( \nu_*^\top Y_t - \hat{\nu}_t^\top Y_t)^2 | \Sigma_*, Y_{1:t-1}]\right]; \\
    \overset{(e)}{=} &\E\left[\frac{1}{ \overline{v}^{-1} + \1^\top \Sigma_*^{-1}\1}\right] + \E\left[ (\nu_* - \hat{\nu}_t)^\top \E \left[ Y_t Y_t^\top | \Sigma_*, Y_{1:t-1} \right] (\nu_* - \hat{\nu}_t) \right]; \\
    = &\E\left[\frac{1}{ \overline{v}^{-1} + \1^\top \Sigma_*^{-1}\1}\right] + \E\left[ (\nu_* - \hat{\nu}_t)^\top (\Sigma_* + \overline{v} \1\1^\top)(\nu_* - \hat{\nu}_t) \right]
\end{align*}
$(a)$ is true using tower property.  $(b)$ uses the law of total variance.  $(c)$ uses the fact that the estimates are drawn from the Gaussian data-generating process.  $(d)$ uses tower property.  $(e)$ is true because $\nu_*$ is determined by $\Sigma_*$ and $\hat{\nu}_t$ is determined by $Y_{1:t-1}$. 

The calculation above implies that
\begin{equation}
    \E[(Z_t - \hat{Z}_t^\pi)^2]=  \E\left[\frac{1}{ \overline{v}^{-1} + \1^\top \Sigma_*^{-1}\1} + (\nu_* - \hat{\nu}_t)^\top (\Sigma_* + \overline{v} \1\1^\top)(\nu_* - \hat{\nu}_t) \right].
\end{equation}
We use Monte-Carlo simulation to estimate the right-hand side. Specifically, we sample a set of $\Sigma_*$, for each $\Sigma_*$ we sample one history $Y_{1:t-1}$, and then compute the quantity inside the expectation.

\subsubsection*{Number of seeds.} 
\begin{itemize}
    \item To perform hyperparameter tuning for EM and predict-each-worker, we sample sixty values of $\Sigma_*$.
    \item To generate plots in Figure~\ref{fig:perf_comp}, we sample fifty values of $\Sigma_*$.  These values are different from the values of $\Sigma_*$ that are used in tuning.
\end{itemize}

\subsection{Metric for Tuning SSL Hyperparameters}
\label{app:ssl_metric}
Here, we discuss the metric that we use for tuning SSL Hyperparameters.  This metric is an expectation of a loss function.  The loss function is equal to the KL divergence between the clairvoyant distribution over $Y_t$ and the distribution implied by the learned SSL models.  For the Gaussian data-generating process, recall that the clairvoyant distribution satisfies $\Pr(Y_t | \theta) \overset{{\rm d}}{=} \normal(0, S_*)$.  This implies that estimating $S_*$ is equivalent to estimating $\Pr(Y_t | \theta)$.  Matrix $S_*$ can be estimated from the learned SSL models, and we provide details of the estimation procedure below.

For the Gaussian data-generating process, our SSL models are parameterized by $\{ \indiWt, \indiLoss\}_{k=1}^K$.  We provide a two-step transformation to go from $\{ \indiWt, \indiLoss\}_{k=1}^K$ to a covariance matrix.  The first step is to apply a function $h: \{\Re^{(K-1)} \times \Re_{+}\}^K \to \Re^{K \times K}$.  Inputs to $h$ are parameters of SSL models, and the output is a matrix.  Suppose we input $\{ u^{(k)}, \ell^{(k)} \}_{k=1}^K$, then the output $h\left(\{ u^{(k)}, \ell^{(k)} \}_{k=1}^K\right)$ satisfies the following.
\begin{align*}
   \left[h\left(\{ u^{(k)}, \ell^{(k)} \}_{k=1}^K\right)\right]_{k,k'}=
    \begin{cases}
    1/\ell^{(k)} & \text{if } k=k'; \\
    -u^{(k,k')}/\ell^{(k)} & \text{if } k\neq k'.
    \end{cases} 
\end{align*}
The output of this function can be a covariance matrix.  For example, if the inputs are the prescient SSL parameters $\{ \indiWtOpt, \indiLossOpt\}_{k=1}^K$, the output is $S_*$.  However, it is not necessary that the output of $h$ will be a covariance matrix.  The second step of our transformation addresses this.

In the second step, we compute a covariance matrix $\hat{S}_t$ that we treat as the estimate of $S_*$.  If the determinant $|h\left(\{ u^{(k)}, \ell^{(k)} \}_{k=1}^K\right)|$ is negative, then we set $\hat{S}_t$ to be the all zeros matrix.  If $\left|h\left(\{ u^{(k)}, \ell^{(k)} \}_{k=1}^K\right)\right|$ is non-negative, then we find a covariance matrix that has the same determinant as the output of $h$ and is closest to this output in 2-norm sense.  We will discuss why such a covariance matrix must exist.  However, before we discuss this, we formally define $\hat{S}_t$. $\hat{S}_t$ is a minimizer of this optimization problem
\begin{align*}
    \min_{S \in \mathcal{S}^K_{++}} &\left\| h\left(\{ u^{(k)}, \ell^{(k)} \}_{k=1}^K\right) - S \right\|_2 \\
    \text{subject to } & |S| = \left|h\left(\{ u^{(k)}, \ell^{(k)} \}_{k=1}^K\right)\right|.
\end{align*}
The optimization problem is feasible because the diagonal matrix with elements $\left\{\left|h\left(\{ u^{(k)}, \ell^{(k)} \}_{k=1}^K\right)\right|,1,\cdots,1\right\}$ satisfies the constraint.

Having computed $\hat{S}_t$, we state the loss function to judge the quality of SSL parameters
\begin{align*}
    \KL\left( \normal(0, S_*) || \normal(0, \hat{S}_t) \right).
\end{align*}

\textbf{Number of seeds.} To perform hyperparameter tuning, we sample sixty values of $S_*$ and average the loss function above to estimate the performance metric.

\subsection{Hyperparameter Tuning with Unobserved Outcomes}
\label{app:hyperparam_tuning_real}

For our simulation study, we tuned hyperparameters assuming the knowledge of the clairvoyant distribution $\Pr(Y_t \in \cdot | \theta)$.  Specifically, we used this knowledge to estimate metrics that judged the quality of SSL models and group estimates.  The clairvoyant distribution $\Pr(Y_t \in \cdot | \theta)$ will be known in any practical situation.  We discuss how to tune hyperparameters in such a situation.  Though our ideas generalize, we make two simplifying assumptions.  First, we assume the data is drawn from a Gaussian data-generating process.  Second, we assume the estimates are zero-mean, i.e., $\E[Y_t] = 0$.

\textbf{Tuning SSL hyperparameters.}  Suppose we have a dataset of estimates $\data$.  We will divide this dataset into a training dataset $\data_{\rm train}$ and a test dataset $\data_{\rm test}$.  Given a set of SSL hyperparameters, first we compute $\{ \indiWt, \indiLoss \}_{k=1}^K$ based on $\data_{\rm train}$ and then use the procedure in Appendix~\ref{app:ssl_metric} to compute $\hat{S}_{\rm train}$ -- the estimate of $S_*$ based on SSL models.  We then compute the sample covariance matrix $\hat{S}_{\rm test}$ based on $\data_{\rm test}$ and judge the SSL hyperparameters based on 
\begin{align*}
    \KL\left(\normal(0, \hat{S}_{\rm train}) || \normal(0, \hat{S}_{\rm test})\right).
\end{align*}
Instead of doing one train-test split, one could also perform the so-called $k$-fold cross-validation.

\textbf{Tuning aggregation hyperparameters.}  The aggregation step has two hyperparameters: $\overline{v}$ and $r$.  Hyperparameter $\overline{v}$ indicates the prior variance $\V[Z_t]$ of true outcomes, and $r$ controls the rate of decay of regularization towards the prior aggregation weights.  We first discuss a method to set $\overline{v}$ and then discuss two methods to tune $r$.  

To set $\overline{v}$, we use the fact that we defined $Z_t$ to be the average of estimates produced by a large population of crowdworkers.  Hence, $\overline{v}$ can be estimated by computing the sample variance of the group estimates produced by averaging.  Ideally, this variance should be computed on a dataset where both $K$ and $t$ are large.  However, this typically would not be practical.  Hence, one would have to choose large enough values of $K$ and $t$ that are also practically feasible.

We now discuss two methods to tune $r$.  The first method involves estimating MSE using a so-called golden dataset.  This dataset contains crowdsourcing tasks for which the true outcomes can be estimated with high accuracy.  For our problem, the golden dataset can be constructed by having many crowdworker estimates for each outcome.  This dataset would only be used for evaluation and not for training.  As it would only be used for evaluation, it should be okay to have a limited number of engagements with each crowdworker when constructing this dataset.  The rationale is that it takes more data to train but less data to evaluate.

The second method to tune $r$ involves estimating a quantity that differs from MSE by an additive constant.  Suppose we want to do an evaluation on crowdsourcing tasks in $\data_{\rm eval}$.  We will first compute a group estimate for each crowdsourcing task in  $\data_{\rm eval}$.  Then, for each crowdsourcing task, we will sample a crowdworker uniformly randomly from a large pool of out-of-sample crowdworkers, get an estimate from the sampled crowdworker, and treat their estimate as the ground truth.  For each task, the crowdworker must be sampled independently.  Formally, this procedure allows us to estimate
\begin{align*}
    \E[( \tilde{Y}_{t,K+1} - \hat{Z}_t^\pi)^2 ],
\end{align*}
where $\tilde{Y}_{t,K+1}$ is the estimate of an out-of-sample crowdworker and $\hat{Z}_t^\pi$ is the group estimate.  The index $K+1$ acts as a placeholder for a random out-of-sample crowdworker, and one should not interpret this as some fixed crowdworker's estimate. 
The squared error above differs from MSE by an additive constant because $\tilde{Y}_{t,K+1}$ is a single-sample Monte-Carlo estimate of $\lim_{K' \to \infty} \frac{\sum_{k=K+1}^{K+K'}\tilde{Y}_{t,k}}{K'}$ and because $Z_t$ is equal to $\lim_{K' \to \infty} \frac{\sum_{k=K+1}^{K+K'}\tilde{Y}_{t,k}}{K'}$ when the limit exists.  We give more details below.
\begin{align*}
    \E[( \tilde{Y}_{t,K+1} - \hat{Z}_t^\pi)^2 ] \overset{(a)}{=} &\E\left[\E\left[( \tilde{Y}_{t,K+1} - \hat{Z}_t^\pi)^2 | Y_{1:t}, \Sigma_* \right]\right] \\
    \overset{(b)}{=} & \E[\V[\tilde{Y}_{t,K+1} | Y_{1:t}, \Sigma_*]] + \E\left[\left( \E\left[\tilde{Y}_{t,K+1} |Y_{1:t}, \Sigma_* \right] - \hat{Z}_t^\pi\right)^2 \right] \\
    \overset{(c)}{=} & \E[\V[\tilde{Y}_{t,K+1} | Y_{1:t}, \Sigma_*]] + \E\left[\left( \E\left[ \lim_{K' \to \infty}\frac{\sum_{k=K+1}^{K' + K}\tilde{Y}_{t,k}}{K'} |Y_{1:t}, \Sigma_* \right] - \hat{Z}_t^\pi\right)^2 \right] \\
    \overset{(d)}{=} & \E[\V[\tilde{Y}_{t,K+1} | Y_{1:t}, \Sigma_*]] + \E\left[\left( \E\left[ Z_t |Y_{1:t}, \Sigma_* \right] - \hat{Z}_t^\pi\right)^2 \right] \\
    = & \E[\V[\tilde{Y}_{t,K+1} | Y_{1:t}, \Sigma_*]] - \E[\V[Z_t | Y_{1:t}, \Sigma_*]] + \E\left[\left( Z_t - \hat{Z}_t^\pi\right)^2 \right].
\end{align*}
Here, $(a)$ is true because of the tower property.  $(b)$ uses the law of total variance.  $(c)$ uses the fact that $(\Tilde{Y}_{t,k}: k=K+1,K+2, \cdots)$ are exchangeable {\it a priori} and remain exchangeable even when conditioned on $Y_{1:t}$ and $\Sigma_*$.  $(d)$ uses the assumption that $\lim_{K' \to \infty} \frac{\sum_{k=K+1}^{K+K'}\tilde{Y}_{t,k}}{K'}$ exists almost surely and $Z_t$ is defined to be equal to this limit when it exists. The final equality makes it clear that $\E[( \tilde{Y}_{t,K+1} - \hat{Z}_t^\pi)^2 ]$ and $\E\left[\left( Z_t - \hat{Z}_t^\pi\right)^2 \right]$, the MSE, differ by an additive constant.

\section{Extension of Predict-Each-Worker for Settings with Observed Outcomes}
\label{app:ext_obs_outcomes}
Here, we describe an extension for our algorithm that can be used when some or all outcomes are observed.  However, if a large fraction of outcomes are observed, then one might be better off using standard supervised learning methods to learn a mapping from estimates to an outcome.

In our extension, we keep the SSL step the same, but we change the aggregation step.  First, we train the SSL models on all estimate vectors.  Then, for all estimate vectors whose outcomes are known, and for each SSL model, we compute the following statistics: expected error and SSL gradient -- this is similar to what we do in Section~\ref{sec:gen_agg}.  These SSL statistics and the corresponding outcomes are then used to train a neural network. This neural network takes a set of expected errors and SSL gradients as input and produces an estimate of the true outcome.  The rationale for using these SSL statistics instead of the estimate vectors as inputs is that it allows leveraging all observed estimates, even the ones for which the outcomes are not observed.  We will denote the neural network obtained after training as $\mathfrak{f}$.

Now, we discuss how to produce a group estimate for a new estimate vector $Y_t$.  We first compute a set of expected errors and SSL gradients based on $Y_t$.  Then, we input these SSL statistics to neural network $\mathfrak{f}$ to produce an estimate $\hat{Z}_{t}^{(\mathfrak{f})} \in \Re$ of the outcome.  Additionally, we use these statistics to compute $\hat{\nu}_t^\top Y_t$, where $\hat{\nu}_t$ is based on our aggregation formula \eqref{eq:agg_formula}.  Finally, we compute the group estimate by taking a convex combination of $\hat{Z}_{t}^{(\mathfrak{f})}$ and $\hat{\nu}_t^\top Y_t$,
\begin{equation*}
    \hat{Z}_t \gets \frac{a}{a + t} \hat{Z}_{t}^{(\mathfrak{f})} + \left(1 - \frac{a}{a + t}\right) \hat{\nu}_t^\top Y_t.
\end{equation*}
Here, $a \in \Re_+$ is a hyperparameter that can be tuned based on a validation set.  We should expect $a$ to increase with $t_o$, the number of estimates for which the outcome is known.  Reason being that larger the value of $t_o$, the more we can trust $\hat{Z}_{t}^{(\mathfrak{f})}$ .

\end{document}